\documentclass[letterpaper]{article} % DO NOT CHANGE THIS
\usepackage{aaai25}  % DO NOT CHANGE THIS
\usepackage{times}  % DO NOT CHANGE THIS
\usepackage{helvet}  % DO NOT CHANGE THIS
\usepackage{courier}  % DO NOT CHANGE THIS
\usepackage[hyphens]{url}  % DO NOT CHANGE THIS
\usepackage{graphicx} % DO NOT CHANGE THIS
\usepackage{booktabs}
\usepackage{amsthm,amsmath,amssymb}
\usepackage{mathrsfs}
\usepackage{natbib}  % DO NOT CHANGE THIS AND DO NOT ADD ANY OPTIONS TO IT
\usepackage{caption} % DO NOT CHANGE THIS AND DO NOT ADD ANY OPTIONS TO IT
\usepackage{algorithm}
\usepackage{algorithmic}
\usepackage{amsmath}
\usepackage{enumerate}
\usepackage{array}
\usepackage{multirow}
\newtheorem{theorem}{Theorem}
\newtheorem{definition}{Definition}

\DeclareMathAlphabet\mathbfcal{OMS}{cmsy}{b}{n}

\usepackage{newfloat}
\usepackage{listings}
\DeclareCaptionStyle{ruled}{labelfont=normalfont,labelsep=colon,strut=off} % DO NOT CHANGE THIS
\floatstyle{ruled}
\newfloat{listing}{tb}{lst}{}
\floatname{listing}{Listing}
\title{OpenViewer: Openness-Aware Multi-View Learning}
\author{
    Shide Du\textsuperscript{\rm 1,}\textsuperscript{\rm 2}, Zihan Fang\textsuperscript{\rm 1,}\textsuperscript{\rm 2}, Yanchao Tan\textsuperscript{\rm 1,}\textsuperscript{\rm 2}, Changwei Wang\textsuperscript{\rm 3}, Shiping Wang\textsuperscript{\rm 1,}\textsuperscript{\rm 2}, Wenzhong Guo\textsuperscript{\rm 1,}\textsuperscript{\rm 2}\thanks{Corresponding author.}\\
}
\affiliations{
    \textsuperscript{\rm 1} College of Computer and Data Science, Fuzhou University, Fuzhou, China\\
    \textsuperscript{\rm 2} Fujian Provincial Key Laboratory of Network Computing and Intelligent Information Processing, Fuzhou University, Fuzhou, China\\
    \textsuperscript{\rm 3} Key Laboratory of Computing Power Network and Information Security, Ministry of Education, Shandong Computer Science Center, Qilu University of Technology, Jinan, China\\
    dushidems@gmail.com, fzihan11@163.com, yctan@fzu.edu.cn, changweiwang@sdas.org, shipingwangphd@163.com, guowenzhong@fzu.edu.cn
}

\usepackage{bibentry}
\begin{document}

\maketitle

\begin{abstract}
Multi-view learning methods leverage multiple data sources to enhance perception by mining correlations across views, typically relying on predefined categories. 
However, deploying these models in real-world scenarios presents two primary openness challenges. 
\textbf{1) Lack of Interpretability:} The integration mechanisms of multi-view data in existing black-box models remain poorly explained; 
\textbf{2) Insufficient Generalization:} Most models are not adapted to multi-view scenarios involving unknown categories. 
To address these challenges, we propose OpenViewer, an openness-aware multi-view learning framework with theoretical support. 
This framework begins with a Pseudo-Unknown Sample Generation Mechanism to efficiently simulate open multi-view environments and previously adapt to potential unknown samples.
Subsequently, we introduce an Expression-Enhanced Deep Unfolding Network to intuitively promote interpretability by systematically constructing functional prior-mapping modules and effectively providing a more transparent integration mechanism for multi-view data. 
Additionally, we establish a Perception-Augmented Open-Set Training Regime to significantly enhance generalization by precisely boosting confidences for known categories and carefully suppressing inappropriate confidences for unknown ones.
Experimental results demonstrate that OpenViewer effectively addresses openness challenges while ensuring recognition performance for both known and unknown samples.
\end{abstract}

%\begin{links}
%     \link{Code}{https://anonymous.4open.science/r/OpenViewer-41E8}
%\end{links}

\section{Introduction}\label{Introduction}

Multi-view learning has emerged as a prominent area of artificial intelligence, focusing on leveraging diverse data sources to enhance perception \cite{Tan24AnEffective, Yu24DVSAI}. 
This learning paradigm processes real-world objects from various extractors or sensors, exploiting correlations across multiple views to enhance performance in applications like computer vision \cite{Ning24Differentiable}, natural language processing \cite{Song24ADualWay}, large-scale language models \cite{Guo23ViewRefer}, and more \cite{Pei23Fewshot, YeL24MileCut}.
%For instance, in image recognition of computer vision, multi-view learning can utilize different visual features, such as color, texture, shape, and deep obtained from various angles or sensors to improve object classification accuracy. 
%Similarly, in text analysis of natural language processing, multi-view learning can integrate textual features extracted from different linguistic perspectives or document representations to enhance tasks like sentiment analysis, document clustering, and so on.
%Additionally, multi-view learning finds application in large-scale language models, where it integrates information from various modalities like text, images, and context to generate more diverse and contextually relevant responses. 
%By considering multiple source data as diversity modalities simultaneously, large-scale language models can better understand user inputs and generate more accurate and contextually appropriate responses, enhancing the overall user experience.
However, traditional multi-view methods, whether heuristic \cite{Zhang23Let, Xiao24MultiViewMaximum} or deep learning \cite{Xu23SelfWeighted, Yang24CapMax}, typically operate under the assumption that all samples belong to known categories \cite{Du23Bridging}. 
When deployed in real-world settings, these approaches encounter two significant openness challenges, as illustrated in Fig. \ref{Framework0}.
\textbf{Challenge \uppercase\expandafter{\romannumeral1}: Lack of Interpretability.} These black-box methods often lack of explanation in the integration process of multi-view data involving both known and unknown category samples.
This opacity undermines their reliability in open scenarios.
%Traditional methods often lack explanation and transparency in how they integrate and discriminate information from multiple views, leading to insufficient trustworthiness when adapting to open scenarios. 
%For instance, in multi-view medical image analysis \cite{Meng22MulViMotion}, a model that combines information from different imaging modalities can provide more comprehensive medical diagnosis, but without interpretable internal mechanisms for how to make diagnosis, it becomes difficult for clinicians to understand the rationale behind these decisions. 
%This may lead to misdiagnosis, especially when using these open-setting models to distinguish between known or unseen conditions.
\textbf{Challenge \uppercase\expandafter{\romannumeral2}: Insufficient Generalization.} Trained on known samples-based datasets, existing multi-view closed-set methods fail to identify unknown categories during testing, frequently mislabeling them as known with unduly high confidences.
Consequently, they struggle to generalize to multi-view environments containing unknown samples. 
This issue arises because the models are not preemptively adapted to the range of potential unknown categories.
%When new categories emerge during testing, models trained on predefined closed-set datasets fail to recognize them and instead, inaccurately classify these as known categories with high confidence. 
%This limitation stems from the models not being pre-adapted to the distribution of potential unknown categories during the training process.
%For instance, in multi-camera action recognition system \cite{Ma24MultiViewTime}, a model trained on a specific set of multi-view object action categories may fail to identify when presented with the action of entirely new objects.
%Given these challenges, there is a pressing need for frameworks that can enhance both the interpretability and generalization of openness-aware multi-view learning.

\begin{figure}[t]
  \centering
  \includegraphics[width=0.48\textwidth]{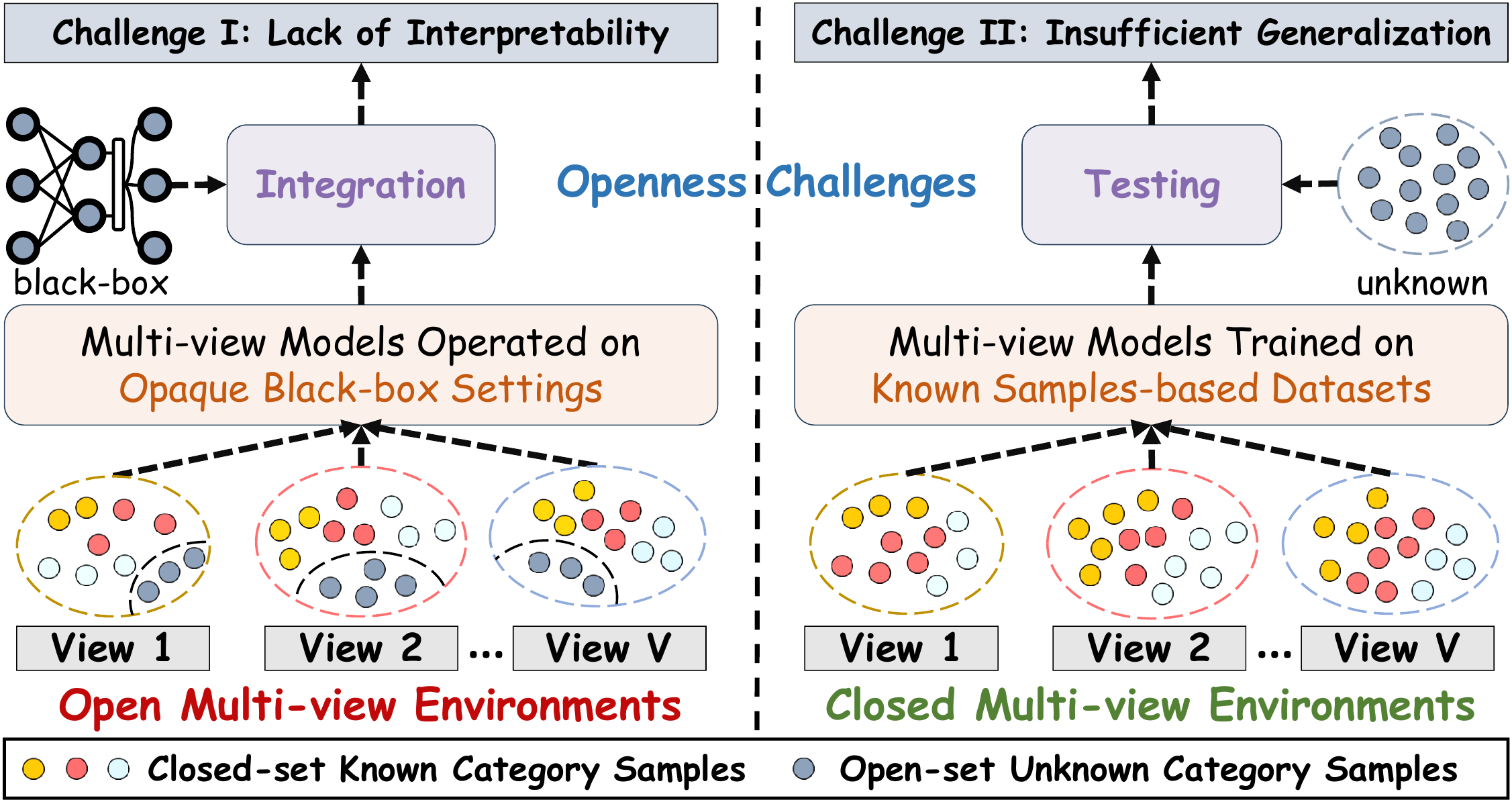}\\
  \caption{Two multi-view environments and challenges.}
  \label{Framework0}
\end{figure}

\begin{figure*}[t]
  \centering
  \includegraphics[width=0.98\textwidth]{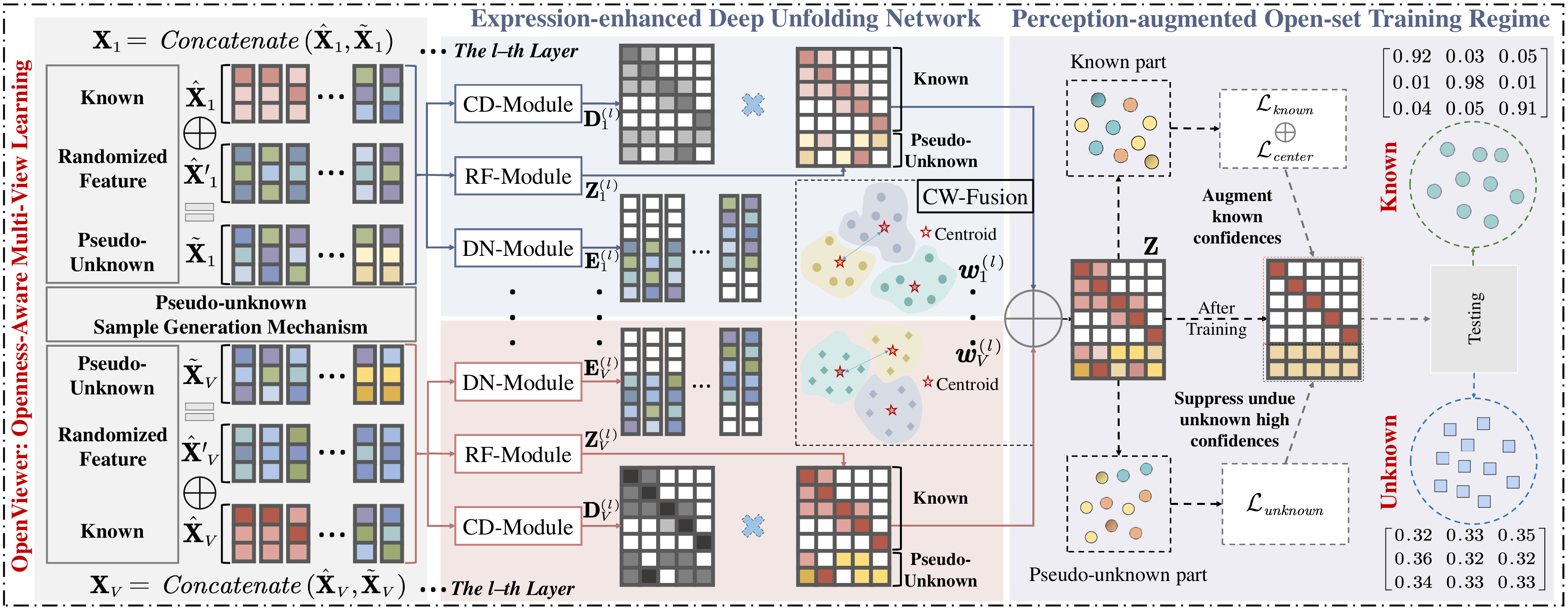}\\
  \caption{An overview of the proposed openness-aware multi-view learning framework (OpenViewer).}
  \label{Framework}
\end{figure*}

To effectively address these challenges, we propose OpenViewer, an openness-aware multi-view learning framework designed for real-world environments, as outlined in Fig. \ref{Framework}.  
OpenViewer starts with a pseudo-unknown sample generation mechanism, allowing the model to efficiently simulate open multi-view environments and previously adapt to potential unknown samples.
%Subsequently, we utilize the Alternating Direction Method of Multipliers (ADMM) optimizer to solve a multi-view feature expression-enhanced objective problem with functionalized prior knowledge, such as sparse redundancy constraints, dictionary consistency learning principles, noise diversity characteristics, and fusion complementarity. 
Grounded on ADMM iterative solutions with functionalized priors, we derive an interpretable multi-view feature expression-enhanced deep unfolding network, comprising redundancy removal, dictionary learning, noise processing, and complementarity fusion modules. 
%derive an interpretable expression-enhanced deep unfolding network with four multi-view functional prior-mapping modules.
The corresponding functions of each module are intuitively reflected in the prior-mapping optimization process, offering a more transparent integration mechanism.
Additionally, we implement a multi-view sample perception-augmented open-set training regime to further boost confidences for known categories and suppress inappropriate confidences for unknown ones.
This enables dynamic perception of known and unknown samples, thereby improving generalization.
Finally, we present theoretical analysis and proof to substantiate OpenViewer’s ability to increase both interpretability and generalization.
%Our experimental results demonstrate that OpenViewer effectively addresses the two openness challenges while ensuring recognition performance for both known and unknown samples. 
%To our knowledge, this is the first openness-aware multi-view learning framework designed for adapting real-world environments for interpretable integration and differentiation of multi-view known and unknown samples.
%We believe this work paves the way for more sophisticated technical designs and empirical evaluations in open-environment multi-view learning.
The main contributions of OpenViewer can be listed as follows:
\begin{itemize}
\item \textbf{\textit{Formulation of OpenViewer:}} We propose OpenViewer, an openness-aware multi-view learning framework designed to tackle the challenges of interpretability and generalization, backed by theoretical guarantees.
\item \textbf{\textit{Openness-aware models design:}} We develop an interpretable expression-enhanced deep unfolding network, bolstered by a pseudo-unknown sample generation mechanism and a perception-augmented open-set training regime, to improve adaptation and generalization.
\item \textbf{\textit{Extensive experiments on real-world datasets:}} Experimental results validate OpenViewer's effectiveness in addressing openness challenges, demonstrating superior recognition performance for both known and unknown.
\end{itemize}

\section{Related Work}\label{RelatedWork}

\subsubsection{Two Multi-view Learning Methods.} %Multi-view learning algorithms can be roughly classified into two categories: 
1) \textbf{Heuristic methods} leverage multi-view prior knowledge to formulate and iteratively solve joint optimization objectives, leading to optimal multi-view learning solutions.
For example, Wan \textit{et al.} \cite{Wan23AutoWeighted} proposed an auto-weighted multi-view optimization problem for large-scale data.
Yu \textit{et al.} \cite{Yu24ANonparametric} devised a non-parametric joint optimization functions to partition multi-view data;
2) \textbf{Deep learning methods} utilize network architectures to automate the optimization of multi-view learning solutions and parameters.
For example, Xiao \textit{et al.} \cite{Xiao23DualFusion} performed multi-view deep learning by the consistency and complementarity.
Xu \textit{et al.} \cite{Xu24DeepVariational} introduced the view-specific encoders and product-of-experts approach to aggregate multi-view information.
Further work on multi-view learning can be discovered in \cite{Chen20Multi, Wang22Alignthen, Yang2022Robust, Liu24Samplelevel} (heuristic) and \cite{Yang21Partially, Lin23Dualcontrastive, Du2024UMCGL, Wang2024Scalable} (deep learning).

\subsubsection{Interpretable Deep Unfolding Networks.} 
%There are currently two mainstream interpretability methods, design-level \cite{Li22Optimization, Hahn23AnInterpretable} (this work also tries to pursue) and post-hoc-level \cite{Abrate21Counterfactual, Lin22OrphicX} interpretability. 
%Deep unfolding networks have achieved success while maintain good interpretability in multiple fields \cite{GregorL10, Bonet22Explaining, Zheng23HybridISTA, Boris24InterpretableNN}.
%They are mapped and derived from iterative solutions obtained from domain-specific problems that encapsulate professional priors and functionalized knowledge, providing the design-level interpretability in network architecture.
Deep unfolding networks, derived from iterative solutions that encapsulate domain-specific priors and functional knowledge, have achieved success while maintaining strong interpretability across multiple fields \cite{GregorL10, Bonet22Explaining, Zheng23HybridISTA, Boris24InterpretableNN}.
Some notable works, for example, Fu \textit{et al.} \cite{Fu22AModelDriven} designed a model-driven deep unfolding structure for JPEG artifacts removal.
Li \textit{et al.} \cite{Li23LRRNet} displayed a low-rank deep unfolding network for hyperspectral anomaly detection. 
Wu \textit{et al.} \cite{Wu2024Designing} constructed a deep unfolding network based on first-order optimization algorithms.
Additional similar efforts in deep unfolding networks can be traced in \cite{Zhou23LearnedImage, Weerdt24DeepUnfolding, Fang2024Beyond}.

\subsubsection{Open-set Learning.} Open-set learning seeks to extend the closed-set hypothesis by equipping models with the ability to distinguish known and unknown classes.
For instance, Dhamija \textit{et al.} \cite{Dhamija2018Reducing} introduced the negative classes for improving the efficiency of unknown rejection.
Duan \textit{et al.} \cite{Duan23Graphanomaly} formulated a subgraph-subgraph contrast to open-set graph learning into a multi-scale contrastive network.
Safaei \textit{et al.} \cite{Bardia2024Entropic} explored an entropic open-set active learning framework to select informative unknown samples. 
Related open-set learning methods can also be found in \cite{Bendale2016Towards, Du23Bridging, Gou24Test}.

%\subsubsection{Limitations and Beyond.}\label{subsec:LimitationsandBeyond}
%The existing work often overlooks developing multi-view methods for open environments, typically addressing either interpretability or generalization, but rarely both.
%This work aims to bridge that gap by designing a multi-view framework that simultaneously enhances both attributes in open environments.
%\textbf{For more related work and their limitations and beyond, kindly see Appendix Section \ref{SupRelatedWork} of Supplementary Materials.}

\section{Openness-Aware Multi-View Learning}\label{TheProposedFramework}
\begin{table}[!htbp]
\centering
\resizebox{0.45\textwidth}{!}{
\begin{tabular}{c||c}
\toprule {Notations} & {Descriptions} \\ 
\midrule
 %  {$\mathbb{R}$}   & The real number space.\\
   {$V$, $C$}   & The number of views and known classes.\\
   {$N^{o}$}   & The number of original training samples.\\
   {$N^{e}$}   & The number of pseudo-unknown training samples, $N^{e} \leq N^{o}$.\\
   {$N$}   & The number of total training samples, $N=N^{o}+N^{e}$.\\
   {$D_{v}$}   & The dimensions of the $v$-th view feature.\\
   % {$C$}   & The number of known classes.\\
%   {$\emph{Pseudo-unknown}$}   & Pseudo-unknown class.\\
%   {$\emph{Unknown}$}   & Unknown class.\\
%   {$\emph{Concatenate}$} & Concatenate operations.\\
  \midrule
  {$\{\hat{\mathbf{X}}_{v}\}_{v=1}^{V}$}   & $ \hat{\mathbf{X}}_{v} \in \mathbb{R}^{N^{o} \times D_{v}}$ is the $v$-view original training samples.\\
   {$\{\tilde{\mathbf{X}}_{v}\}_{v=1}^{V}$}   & $\tilde{\mathbf{X}}_{v}\in \mathbb{R}^{N^{e} \times D_{v}}$ is the $v$-view pseudo-unknown training samples.\\
   {$\{\mathbf{X}_{v}\}_{v=1}^{V}$}   & $\mathbf{X}_{v} \in \mathbb{R}^{N \times D_{v}} = \emph{Concatenate}(\hat{\mathbf{X}}_{v}, \tilde{\mathbf{X}}_{v})$ is the $v$-view training samples.\\
    {$\hat{\mathbf{Y}}$}   & $ \hat{\mathbf{Y}} \in \{1, 2, \cdots, C\}$ is the original training labels.\\
   {$\tilde{\mathbf{Y}}$}   & $ \tilde{\mathbf{Y}} \in \{C+1\}$ is the pseudo-unknown training labels.\\
   {$\mathbf{Y}$}  & $ \mathbf{Y} \in \{1, 2, \cdots, C, C+1\} $ is the total training labels.\\
%   {$\mathbf{Y}_{test}$}  & $ \mathbf{Y}_{test} \in \{1, 2, \cdots, C, C+1\} $ is the total test labels.\\
   \midrule
    {$\{\mathbf{Z}_{v}\}_{v=1}^{V}$}   & $\mathbf{Z}_{v} \in \mathbb{R}^{N \times C}$ is the $v$-th redundancy free representation.\\
   {$\{\mathbf{D}_{v}\}_{v=1}^{V}$}   & $\mathbf{D}_{v} \in \mathbb{R}^{C \times D_{v}}$ is the $v$-th consistency dictionary learning matrix.\\
   {$\{\mathbf{E}_{v}\}_{v=1}^{V}$}   & $\mathbf{E}_{v} \in \mathbb{R}^{N \times D_{v}}$ is the $v$-th diversity noise processing matrix.\\
   {$\{\mathbf{w}_{v}\}_{v=1}^{V}$}   & $\mathbf{w}_{v}$ is the $v$-th complementarity fusion contribution weight.\\
   {$\mathbf{Z}$}   & $\mathbf{Z} \in \mathbb{R}^{N \times C}$ is the complementarity enhanced representation.\\
   \midrule
% {$\mathcal{D}_{total}$}   & The whole multi-view dataset $\mathcal{D}_{total}=\mathcal{D}_{train}\cup\mathcal{D}_{valid}\cup\mathcal{D}_{test}$.\\
 {$\mathcal{D}_{train}$}   & Multi-view training set,  $\mathcal{D}_{train}=\mathcal{D}_{original}\cup\mathcal{D}_{generated}$.\\
 %{$\mathcal{D}_{valid}$, $\mathcal{D}_{test}$}   & The multi-view valid and test set.\\
 {$\mathcal{D}_{original}$}   & Original training set, $\mathcal{D}_{original}=\{\{\hat{\mathbf{X}}_{v}\}_{v=1}^{V}, \hat{\mathbf{Y}}\}$.\\
  {$\mathcal{D}_{generated}$}   & Pseudo-unknown training set, $\mathcal{D}_{generated}=\{\{\tilde{\mathbf{X}}_{v}\}_{v=1}^{V}, \tilde{\mathbf{Y}}\}$.\\
  \bottomrule
\end{tabular}}
\caption{Essential notations and descriptions.}
\label{SymbolicNormalization}
\end{table}
\subsection{Pseudo-unknown Sample Generation Mechanism}\label{Pseudounknown} 
%The essential notations and their definitions are provided in Table \ref{SymbolicNormalization} for reference.
The necessary notations are first listed in Table \ref{SymbolicNormalization}.
%Currently, mainstream multi-view learning methods focus on data mining in closed-set environments. 
To better tackle open-set environments, inspired by Mixup \cite{ZhangCDL18}, we utilize a pseudo-unknown sample generation mechanism to efficiently simulate open multi-view environments and previously adapt to potential unknown samples.
Specifically, a perturbation parameter $\zeta \in [0, 1]$ is sampled from a $\zeta \sim \mathbf{Beta}(\omega, \omega)$ distribution. 
Based on this, we randomly select $\hat{\boldsymbol{x}}^{(i)}_{v}$ and $\hat{\boldsymbol{x}}'^{(j)}_{v}$ from the $v$-th view original feature $\hat{\mathbf{X}}_{v}$, ensuring they belong to different categories.  
%To generate the pseudo-unknown samples, we randomly select a sample $\hat{\boldsymbol{x}}^{(j)}_{v}$, while another sample $\hat{\boldsymbol{x}'}^{(j)}_{v}$ is randomly selected from another class.
Then, the pseudo-unknown sample $\tilde{\boldsymbol{x}}_{v}$ is generated as
\begin{equation}
\begin{array}{ll}\label{genPseudo}
\tilde{\boldsymbol{x}}_{v}= \zeta\hat{\boldsymbol{x}}^{(i)}_{v}+\left(1-\zeta\right)\hat{\boldsymbol{x}}'^{(j)}_{v},
\end{array}
\end{equation}
where $\zeta$ determines the extent to which each original sample contributes to the features of the generated samples.
Using Eq. \eqref{genPseudo}, we generate a set of pseudo-unknown samples, $\mathcal{D}_{generated}$, and merge it with $\mathcal{D}_{original}$ to prepare the model for adapting to unknown classes.

%\subsubsection{Generalized Multi-view Problem with Mixed Samples.}\label{ExpressionenhancedObjectiveProblem}
%\begin{equation}\label{q1}
%\begin{array}{ll}
%\mathop{\min}\limits_{\mathbf{Z}_{v}}\sum\limits_{{v=1}}^{V}\Big(\mathcal{I}(\mathbf{X}_{v}, \mathbf{Z}_{v}, \mathbf{D}_{v})+\alpha\mathbf{\Omega}(\mathbf{Z}_{v})\Big),
%\end{array}
%\end{equation}

%\begin{equation}\label{q2}
%\begin{array}{ll}
%\mathop{\min}\limits_{\mathbf{Z}_{v}, \mathbf{D}_{v}}\sum\limits_{{v=1}}^{V}\Big(\mathcal{I}(\mathbf{X}_{v}, \mathbf{Z}_{v}, \mathbf{D}_{v})+\alpha\mathbf{\Omega}(\mathbf{Z}_{v})+\beta\mathbf{\Psi}(\mathbf{D}_{v})\Big),
%\end{array}
%\end{equation}
\subsection{Expression-enhanced Deep Unfolding Network}\label{InterpretableDeepNetwork}

Subsequently, we design an interpretable expression-enhanced deep unfolding network to clarify the multi-view integration principle.
We first abstract four multi-view functionalized priors as shown in Fig. \ref{Framework2}, including:
1) \textbf{View-specific Redundancy} denotes the redundant similar features within each view;
2) \textbf{View-specific Consistency} indicates the dictionary coefficients, reflecting each representation's consistent contribution to the reconstruction of each view;
3) \textbf{View-specific Diversity} signifies the diverse noise information within each view;
4) \textbf{Cross-view Complementarity} refers to processed cross-view representations that can complementary, enhance and express each other.
Following that, we first consider the three view-specific priors, and construct a generalized expression-enhanced optimization problem as
\begin{equation}\label{q3}
\begin{array}{ll}
\mathop{\min}\limits_{\mathbf{Z}_{v}, \mathbf{D}_{v}, \mathbf{E}_{v}}\sum\limits_{{v=1}}^{V}\Big(\mathcal{I}(\mathbf{X}_{v}, \mathbf{Z}_{v}, \mathbf{D}_{v}, \mathbf{E}_{v})+\alpha\mathbf{\Omega}(\mathbf{Z}_{v})\\~~~~~~~~~~~~~~~~~~~~~~~~+\beta\mathbf{\Psi}(\mathbf{D}_{v})+\gamma\mathbf{\Phi}(\mathbf{E}_{v})\Big),
\end{array}
\end{equation}
where $\alpha$, $\beta$, $\gamma$ are the regularization parameters, and generalized Problem \eqref{q3} includes the above functionalized priors that can be further concretized as
\begin{equation}\label{openproblem}
\begin{array}{ll}
\mathop{\min}\limits_{\mathbf{Z}_{v}, \mathbf{D}_{v}, \mathbf{E}_{v}}\sum\limits_{{v=1}}^{V}\Big(\frac{1}{2}\Vert\mathbf{X}_{v}-\mathbf{Z}_{v}\mathbf{D}_{v}-\mathbf{E}_{v}\Vert_{F}^{2}+\alpha \Vert\mathbf{Z}_{v}\Vert_{1} \\~~~~~~~~~~~~~~~~~~~~~~~~+\frac{\beta}{2}\Vert\mathbf{D}_{v}\Vert_{F}^{2}+\gamma \Vert\mathbf{E}_{v}\Vert_{2, 1}\Big).
\end{array}
\end{equation}

\begin{figure}[t]
  \centering
  \includegraphics[width=0.48\textwidth]{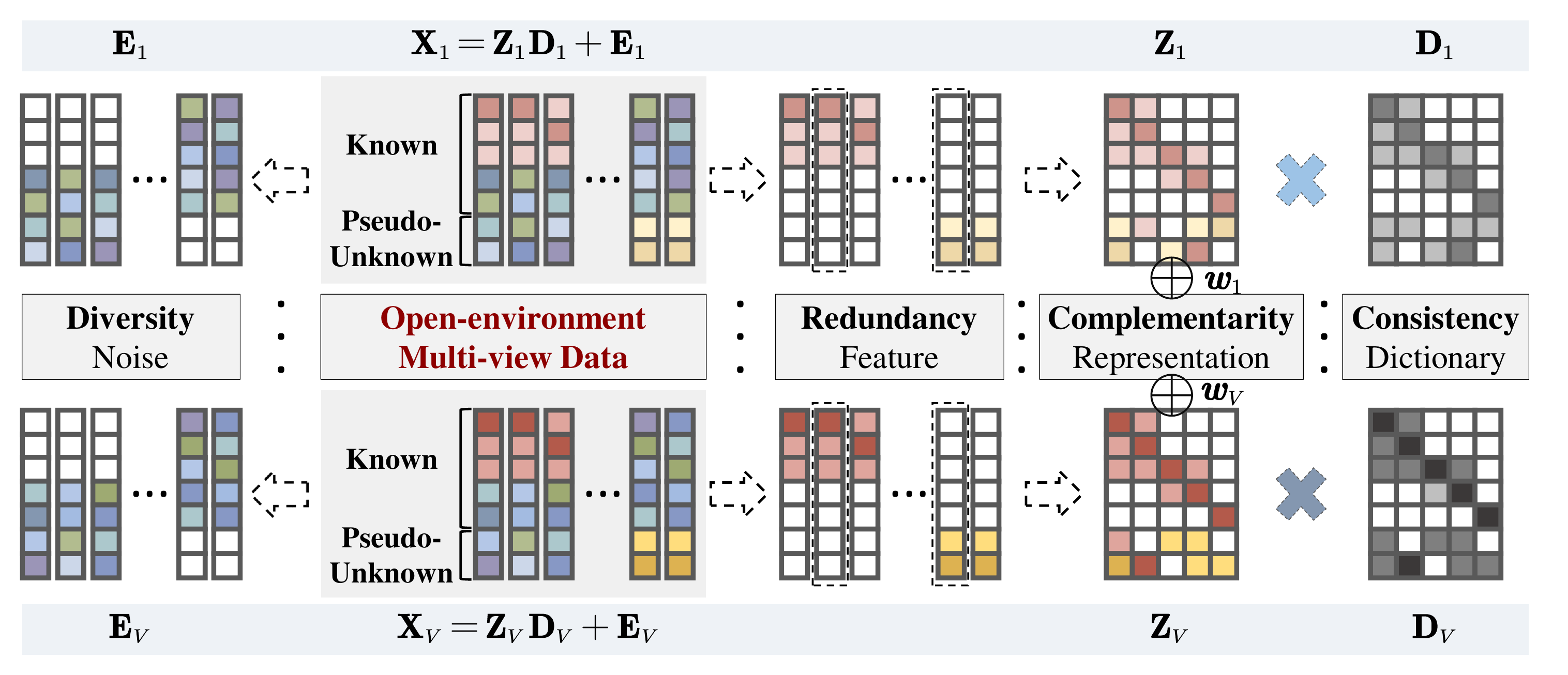}\\
  \caption{Four multi-view priors and their relationships.}
  \label{Framework2}
\end{figure}
Problem \eqref{openproblem} aims to learn a \textit{redundancy} free representation $\mathbf{Z}_{v}$ using $l_1$-norm $\Vert \cdot \Vert_{1}$, while optimizing \textit{consistency} dictionary coefficients $\mathbf{D}_{v}$ and capturing \textit{diversity} noise $\mathbf{E}_{v}$ with the $l_{2, 1}$-norm $\Vert \cdot \Vert_{2, 1}$.
So $\mathbf{X}_v$ can be expressed as a linear combination $\mathbf{Z}_v\mathbf{D}_v+\mathbf{E}_v$.
To optimize such a mixed non-convex problem \eqref{openproblem} consisting of smooth terms $\mathcal{I}(\cdot), \mathbf{\Psi}(\cdot)$ and non-smooth terms $\mathbf{\Omega}(\cdot), \mathbf{\Phi}(\cdot)$, ADMM \cite{Boyd2011Distributed} is employed to decompose it into three sub-problems for solving.
For $\mathbf{Z}_{v} = \{\hat{\mathbf{Z}}_{v}, \tilde{\mathbf{Z}}_{v}\}$, $\mathbf{E}_{v} = \{\hat{\mathbf{E}}_{v}, \tilde{\mathbf{E}}_{v}\}$, $\mathbf{D}_{v} = \{\hat{\mathbf{D}}_{v}, \tilde{\mathbf{D}}_{v}\}$ sub-problems, where $\hat{\mathbf{Z}}_{v}$, $\hat{\mathbf{E}}_{v}$, and $\hat{\mathbf{D}}_{v}$ are known, and $\tilde{\mathbf{Z}}_{v}$, $\tilde{\mathbf{E}}_{v}$, and $\tilde{\mathbf{D}}_{v}$ are pseudo-unknown, we utilize proximal gradient descent method \cite{BeckTeboulle09AFast} to solve $\mathbf{Z}_{v}$ and $\mathbf{E}_{v}$ variables, while $\mathbf{D}_{v}$ variable has a closed-form solution, obtained as
\begin{equation}\label{t3}
\begin{aligned}
\mathbf{Z}_{v}^{(l+1)}&\leftarrow\mathbfcal{S}_{\frac{\alpha}{L_{p_{v}}}}\left(\mathbf{Z}_{v}^{(l)}-\frac{1}{L_{p_{v}}}\nabla \mathcal{I}(\mathbf{Z}_{v}^{(l)})\right), \\ \mathbf{E}_{v}^{(l+1)}&\leftarrow\mathbfcal{P}_{\frac{\gamma}{L_{p_{v}}}}\left(\mathbf{E}_{v}^{(l)}-\frac{1}{L_{p_{v}}}\nabla \mathcal{I}(\mathbf{E}_{v}^{(l)})\right), \\
\mathbf{D}_{v}^{(l+1)} &\leftarrow \{\nabla\mathcal{I}(\mathbf{D}_{v}^{(l)})+\nabla\mathbf{\Psi}(\mathbf{D}_{v}^{(l)}) = 0\},
\end{aligned}
\end{equation}
where $\mathbfcal{S}_{\frac{\alpha}{L_{p_{v}}}}(\cdot)$ and $\mathbfcal{P}_{\frac{\gamma}{L_{p_{v}}}}(\cdot)$ are the redundancy and diversity proximal operators, respectively.
$\nabla(\cdot)$ denotes the gradient of the current variable, $L_{p_{v}}$ is the $v$-th Lipschitz constant of $\nabla\mathcal{I}(\cdot)$, and $l$ is the current iteration.
Subsequently, we expand the gradient-related notations, detailed as
\begin{equation}\label{unfolding}
\left\{
\begin{array}{ll}
\mathbf{Z}_{v}^{(l+1)} & \leftarrow \mathbfcal{S}_{\frac{\alpha}{L_{p_{v}}}}\Big(\mathbf{Z}_{v}^{(l)}-\frac{1}{L_{p_{v}}}(\mathbf{Z}_{v}^{(l)}\mathbf{D}_{v}^{(l)}(\mathbf{D}_{v}^\top)^{(l)}\\& ~~~~~- \mathbf{X}_{v}(\mathbf{D}_{v}^\top)^{(l)}+\mathbf{E}_{v}^{(l)}(\mathbf{D}_{v}^\top)^{(l)})\Big), \\
& \Rightarrow \mathbfcal{S}_{\frac{\alpha}{L_{p_{v}}}}\Big(\mathbf{Z}_{v}^{(l)}(\mathbf{I}-\frac{1}{L_{p_{v}}}\mathbf{D}_{v}^{(l)}(\mathbf{D}_{v}^\top)^{(l)}) \\&~~~~~ + \frac{1}{L_{p_{v}}}(\mathbf{X}_{v}-\mathbf{E}_{v}^{(l)})(\mathbf{D}_{v}^\top)^{(l)}\Big), \\ 
\mathbf{D}_{v}^{(l+1)} & \leftarrow \Big((\mathbf{Z}_{v}^\top)^{(l+1)}\mathbf{Z}_{v}^{(l+1)}+\beta\mathbf{I}\Big)^{-1}(\mathbf{Z}_{v}^\top)^{(l+1)}\\&~~~~~(\mathbf{X}_{v}-\mathbf{E}_{v}^{(l)}), \\ 
\mathbf{E}_{v}^{(l+1)} & \leftarrow \mathbfcal{P}_{\frac{\gamma}{L_{p_{v}}}}\left(\mathbf{X}_{v} - \mathbf{Z}_{v}^{(l+1)}\mathbf{D}_{v}^{(l+1)}\right),
\end{array}
\right.
\end{equation}
where $\mathbf{I} \in \mathbb{R}^{C \times C}$ is an identity matrix.
Thus far, we have used the ADMM optimizer to solve the corresponding sub-problems and derive iterative solutions for \textit{redundancy}, \textit{consistency}, \textit{diversity} priors.
%Ultimately, to achieve the fusion of \textit{complementary} features processed, the usual approach is to perform weighted averaging fusion on them as
%\begin{equation}\label{ZVFusion}
%\mathbf{Z}^{(l+1)} \leftarrow \frac{1}{V}\sum\limits_{{v=1}}^{V}({\mathbf{Z}}_{1}^{(t+1)}, \mathbf{Z}_2^{(l+1)}, \cdots, {\mathbf{Z}}_{v}^{(l+1)}).
%\end{equation}
%However, relying solely on weighted averaging to fuse this information may dilute or ignore known discrimination or unknown recognition features, thereby affecting the performance of multi-view learning.
Then, an inter-class discretion-guided weighting method is applied to account for the cross-view \textit{complementary} prior.
Intuitively, the closer the sample centroids are within a view, the less complementary information each view provides.
%Intuitively, the proximity of sample centroids within a view is negatively correlated with the complementary information of each view, as the more concentrated the sample, the lower the complementary information provided.
Base on this, we dynamically perceive the centroid as $\boldsymbol{o}^{(i)}_{v}=\frac{1}{\left|B_i\right|} \sum\limits_{j:\mathbf{Y}_j=i} \boldsymbol{z}^{(j)}_{v}$ and then calculate the inter-class distances between all centroids for each view, where $\left|B_i\right|$ is the number of training instances in category $i$.
To prevent the largest distance between two classes from overly influencing the measure of inter-class discrepancy, we focus on the minimum distance $\boldsymbol{d}_v$ between any two categories, defined as $\boldsymbol{d}_v=\min\left\{\operatorname{Dist}\left(\boldsymbol{o}_{v}^{(i)}, \boldsymbol{o}_{v}^{(j)}\right)\right\}, i, j \in \mathbf{Y} \text{ and } i \neq j$.
Here, $\operatorname{Dist}\left(\cdot, \cdot\right)$ is the distance function, which in this work is the Euclidean distance.
This strategy provides a balanced assessment of the complementary contribution of different views.
The higher the complementarity information between views (\emph{i.e.,} the greater the distance between centroids), the greater their assigned weights, denoted as
\begin{equation}
\begin{array}{ll}\label{fusion3}
\mathbf{w}_v=\frac{\exp\left(-\overline{\boldsymbol{d}}_v\right)}{\sum_{v=1}^V{\exp\left(-\overline{\boldsymbol{d}}_v\right)}}, \overline{\boldsymbol{d}}_v=\boldsymbol{d}_v^{-1}/\left\|{\boldsymbol{d}_v^{-1}}\right\|_1,  \sum\limits_{v=1}^V \mathbf{w}_v=1,
\end{array}
\end{equation}
where $\overline{\boldsymbol{d}}_v$ is obtained by normalizing the inverse of $\boldsymbol{d}_v$ between centroids through $\ell_1$-norm.
At last, the inter-class discretion-guided weights $\{\mathbf{w}_{v}\}_{v=1}^{V}$ can be applied to perform complementary fusion $\mathbfcal{F}(\cdot)$ as
\begin{equation}\label{ZWFusion}
\mathbf{Z}^{(l+1)}\leftarrow\sum\limits_{v=1}^V \mathbf{w}_{v}^{(l+1)} \mathbf{Z}_v^{(l+1)}.
\end{equation}
%Eq. \eqref{ZWFusion} fuses different known and unknown information by measured complementarity weights of different views to form an overall fusion representation $\mathbf{Z} = \{\hat{\mathbf{Z}}, \tilde{\mathbf{Z}}\}$ with enhanced feature expression.

Based on solutions \eqref{t3} and \eqref{ZWFusion}, the multi-view feature expression-enhanced deep unfolding network can be conceptualized as four interpretable prior-mapping modules by parameterizing alternative components \cite{Zhou23LearnedImage, Weerdt24DeepUnfolding} as 
\begin{equation}\label{stotalnet2}
\left\{\begin{array}{ll}
\textbf{RF-Module: } \mathbf{Z}_{v}^{(l+1)}&\leftarrow\mathbfcal{S}_{\theta_{v}^{(l)}}\Big(\mathbf{Z}_{v}^{(l)}\mathbf{R}+(\mathbf{X}_{v}-\mathbf{E}_{v}^{(l)})\\&~~~~~~~~~~~~~~~(\mathbf{D}_{v}^\top)^{(l)}\mathbf{U}\Big),\\
\textbf{CD-Module: } \mathbf{D}_{v}^{(l+1)}&\leftarrow\mathbf{M}(\mathbf{Z}_{v}^\top)^{(l+1)}(\mathbf{X}_{v}-\mathbf{E}_{v}^{(l)}),\\
\textbf{DN-Module: } \mathbf{E}_{v}^{(l+1)}&\leftarrow\mathbfcal{P}_{\rho_{v}^{(l)}}\left(\mathbf{X}_{v}- \mathbf{Z}_{v}^{(l+1)}\mathbf{D}_{v}^{(l+1)}\right),\\ \textbf{CW-Fusion: } \mathbf{Z}^{(l+1)}&\leftarrow\sum\limits_{v=1}^V \mathbf{w}_{v}^{(l+1)} \mathbf{Z}_v^{(l+1)},
\end{array}\right.
\end{equation}
where $\mathbf{R} = \mathbf{I}-\frac{1}{L_{p_{v}}}\mathbf{D}_{v}\mathbf{D}_{v}^\top$, $\mathbf{U} = \frac{1}{L_{p_{v}}}\mathbf{I}$, and $\mathbf{M} = (\mathbf{Z}_{v}^\top\mathbf{Z}_{v}+\beta\mathbf{I})^{-1}$. 
The learnable redundancy and diversity proximal operators $\mathbfcal{S}_{\theta_{v}^{(l)}}(\cdot)$ and $\mathbfcal{P}_{\rho_{v}^{(l)}}(\cdot)$ are the reparameterized versions of $\mathbfcal{S}_{\frac{\alpha}{L_{p_{v}}}}(\cdot)$ and $\mathbfcal{P}_{\frac{\gamma}{L_{p_{v}}}}(\cdot)$, with learnable threshold parameters $\theta_{v}$ and $\rho_{v}$, respectively.
Moreover, $\mathbfcal{S}_{\theta_{v}}(\boldsymbol{a}^{(ij)})=\sigma(\boldsymbol{a}^{(ij)}-\theta_{v})-\sigma(-\boldsymbol{a}^{(ij)}-\theta_{v})$, and $\mathbfcal{P}_{\rho_{v}}(\boldsymbol{a}^{(i)}) = \frac{\sigma(\left\|\boldsymbol{a}^{(i)}\right\|_2-\rho_{v})}{\left\|\boldsymbol{a}^{(i)}\right\|_2} \boldsymbol{a}^{(i)}$, 
if $\rho_{v}<\left\|\boldsymbol{a}^{(i)}\right\|_2$; otherwise, $0$. 
$\boldsymbol{a}^{(ij)}$ is the element in the $i$-th row and $j$-th column of the matrix, $\boldsymbol{a}^{(i)}$ is the $i$-th column of the matrix, and $\sigma(\cdot)$ can be activation functions such as ReLU, SeLU and etc.

The constructed network, incorporating these modules, can engage in multi-view expression enhancement while integrating their functions into deep networks to maintain interpretability:
1) \textbf{Redundancy Free Representation Module (RF-Module)} introduces learnable layers and redundancy-free operators to reduce redundant features and retain the most critical view information $\mathbf{Z}_{v}$;
2) \textbf{Consistency Dictionary Learning Module (CD-Module)} captures dictionary coefficients $\mathbf{D}_{v}$ within each view, denoting the consistent contribution of each $\mathbf{Z}_{v}$ to the reconstruction of $\mathbf{X}_{v}$; 
3) \textbf{Diversity Noise Processing Module (DN-Module)} develops learnable diversity operators to eliminate irrelevant information $\mathbf{E}_{v}$ caused by the noise or outliers; 
4) \textbf{Complementarity Fusion Representation Module (CW-Fusion)} implements complementary weight fusion to integrate representations as $\mathbf{Z}$ to differentiate between known and unknown. 
Unfolding network \eqref{stotalnet2} is composed of $L$ layers, with each layer corresponding to a single ADMM iteration.
The interpretability is reflected in the optimization process: 1) For multi-view known parts, it enhances expression by processing noise and integrating complementary; 
2) For multi-view unknown parts, it employs redundancy removal, noise processing, and adapts to a pseudo-unknown dictionary to highlight inappropriate unknown confidences. 
This enhanced expression provides a solid foundation for distinguishing between known and unknown, thereby boosting OpenViewer's interpretability and trustworthiness.
%This enhanced expression provides a solid foundation for distinguishing between known and unknown classes in open settings, thereby boosting OpenViewer's interpretability and trustworthiness.

\subsection{Perception-augmented Open-set Training Regime}\label{TrainingRegimes}
The above interpretable network has performed feature-level integration and enhancement.
Subsequently, we design a loss regime to further augment sample-level perception and improve the model's generalization.
%\subsubsection{Known Sample Perception-augmented Loss.}\label{KnownInstance}
For known samples, we first ensure the model's ability to recognize them by applying a cross-entropy loss. 
Building on this, we promote the separation of all known classes by a distance margin term $\max(\boldsymbol{\xi}-\|\hat{\boldsymbol{z}}^{(i)}\|_{2}, 0)^2$, formalized as
\begin{equation}
\begin{array}{ll}\label{lossCla1}
 \mathcal{L}_{known} &= -\frac{1}{N^{o}}\sum\limits_{i=1}^{{N}^{o}}\sum\limits_{c=1}^{C} \left(\hat{\boldsymbol{y}}_{c}^{(i)}\log P(c \mid \hat{\boldsymbol{z}}^{(i)})\right)\\&+\sum\limits_{i=1}^{{N}^{o}}\max(\boldsymbol{\xi}-\|\hat{\boldsymbol{z}}^{(i)}\|_{2}, 0)^2,
\end{array}
\end{equation}
where $\hat{\boldsymbol{z}}^{(i)} \in \hat{\mathbf{Z}}$, $P$ is the $\operatorname{Softmax}$ score, and $\boldsymbol{\xi}$ is the distance margin.
%The first term of Loss \eqref{lossCla1} ensures that the model effectively identifies known samples by minimizing the cross-entropy. 
%The second penalty term of Loss \eqref{lossCla1} applies to samples whose representation vector norm is less than or equal to $\boldsymbol{\xi}$. 
In this way, the feature vector is pushed out of the margin $\boldsymbol{\xi}$ to make its norm as close to or greater than $\boldsymbol{\xi}$ as possible, thereby augmenting the discrimination between known class samples.
%Loss \eqref{lossCla1} ensures that the model has more easily identifiable features and smaller entropy for known, making it more responsive to known.
%\subsubsection{Unknown Sample Perception-augmented Loss.}\label{UnknownInstance}
For the more critical unknown part, we aim to minimize the allocation of pseudo-unknown samples to known groups.
%This means that the entropy of such unknown samples for each known class is average, so it will be enforced that it does not belong to any known class.
Therefore, we employ a $\ell_2$-norm regularization term $\|\tilde{\boldsymbol{z}}^{(i)}\|_{2}^{2}$ to ensure that OpenViewer suppresses excessive unknown high confidences, expressed as
\begin{equation}
\begin{aligned}\label{lossCla3}
 \mathcal{L}_{unknown} = -\frac{1}{C}\sum\limits_{i=1}^{{N}^{e}}\sum\limits_{c=1}^{C}\left(\log P(c \mid \tilde{\boldsymbol{z}}^{(i)}) \right)+\sum\limits_{i=1}^{{N}^{e}}\|\tilde{\boldsymbol{z}}^{(i)}\|_{2}^{2},
\end{aligned}
\end{equation}
where $\tilde{\boldsymbol{z}}^{(i)} \in \tilde{\mathbf{Z}}$, and $P(c \mid \tilde{\boldsymbol{z}}^{(i)})$ is the probability that the model predicts the pseudo-unknown sample $\tilde{\boldsymbol{z}}^{(i)}$ as belonging to category $c$.
Loss \eqref{lossCla3} ensures that the model's prediction confidence for each known class is average and penalizes pseudo-unknown that are close to known, thereby suppressing inappropriate confidences for unknown samples.
%Loss \eqref{lossCla3} ensures that the model responds to unknown samples as low as possible.
%\begin{equation}
%\begin{aligned}\label{lossCla3v}
% (\mathcal{L}_{unknown})_{v}= \sum\limits_{v=1}^{V}\underbrace{-\frac{1}{|\hat{\mathbf{Y}}|}\sum\limits_{(\{\hat{\boldsymbol{x}}_{v}^{(i)}\}_{v=1}^{V}, \hat{\boldsymbol{y}}^{(i)}) \in \mathcal{D}_{original}}\left(\mathbb{I}^{\hat{\boldsymbol{y}}^{(i)}}\log P(\hat{\boldsymbol{y}}^{(i)} \mid \tilde{\boldsymbol{z}}_{v}^{(i)}) \right)}_{(\{\tilde{\boldsymbol{x}}_{v}^{(i)}\}_{v=1}^{V}, \tilde{\boldsymbol{y}}^{(i)}) \in \mathcal{D}_{generated} \cap \tilde{\boldsymbol{z}}_{v}^{(i)}\leftarrow f(\{\tilde{\boldsymbol{x}}_{v}^{(i)}\}_{v=1}^{V})}+\mu\|\tilde{\boldsymbol{z}}_{v}^{(i)}\|^{2},
%\end{aligned}
%\end{equation}
%{(\tilde{x}^{(i)}, \tilde{y}^{(i)}) \in \mathcal{D}_{pseudo}}

%\subsubsection{Center Sample Perception-augmented Loss.}\label{CenterSample}
However, the above loss only increases the inter-class separability between known and unknown samples. 
To promote intra-class compactness, we use the following center loss to further separate the feature vectors of different classes as
\begin{equation}
\begin{aligned}\label{lossCla4}
\mathcal{L}_{center}=\frac{1}{2} \sum_{i=1}^{N^{o}}\left\|\hat{\boldsymbol{z}}_{c}^{(i)}-\boldsymbol{c}^{\hat{\boldsymbol{y}}_{c}^{(i)}}\right\|_{2}^{2},
\end{aligned}
\end{equation}
where $\boldsymbol{c}^{\hat{\boldsymbol{y}}_{c}^{(i)}}$ is the center vector corresponding to the $i$-th sample's true label.
Then, each category center is dynamically updated during training to better reflect the sample distribution of its corresponding category, described as
\begin{equation}
\begin{aligned}\label{lossCla6}
\Delta \boldsymbol{c}^{(j)\in C}=\frac{\sum_{i=1}^{N^{o}} \delta\left(\hat{\boldsymbol{y}}_{c}^{(i)}=j\right) \cdot\left(\boldsymbol{c}^{(j)}-\hat{\boldsymbol{z}}_{c}^{(i)}\right)}{1+\sum_{i=1}^{N^{o}} \delta\left(\hat{\boldsymbol{y}}_{c}^{(i)}=j\right)},
\end{aligned}
\end{equation}
where $\Delta \boldsymbol{c}^{(j)\in C}$ the update amount for the center of category $j$, $\delta(\cdot)$ is an indicator function that takes the value $1$ when the $i$-th sample belongs to category $j$; otherwise, $0$.
%The upper part of Eq. \eqref{lossCla6} calculates the sum of deviations between all samples belonging to category $j$ and their current center vectors, while the lower part of Eq. \eqref{lossCla6} ensures the stability of the update process.
Meanwhile, the center vectors of each class are adjusted by the calculated update amounts as $\boldsymbol{c}^{\hat{\boldsymbol{y}}_{c}^{(i)}}_{new}\leftarrow \boldsymbol{c}^{\hat{\boldsymbol{y}}_{c}^{(i)}}_{old}-\Delta \boldsymbol{c}^{(j)}$, where $\boldsymbol{c}^{\hat{\boldsymbol{y}}_{c}^{(i)}}_{new}$ is the updated center vectors, while $\boldsymbol{c}^{\hat{\boldsymbol{y}}_{c}^{(i)}}_{old}$ is the old center vectors.
%By combining Eqs. \eqref{lossCla4}, \eqref{lossCla6}, and \eqref{lossClacnew}, it is ensured that the class center can gradually adjust to the optimal position and augment intra-class compactness.
%\begin{equation}
%\begin{aligned}\label{lossCla5}
%\frac{\partial \mathcal{L}_{center}}{\partial %\hat{\boldsymbol{z}}^{(i)}}=\hat{\boldsymbol{z}}^{(i)}-%\boldsymbol{c}^{\hat{\boldsymbol{y}}^{(i)}}, 
%\end{aligned}
%\end{equation}

At last, we train the unfolding network \eqref{stotalnet2} by combining these losses to augment perception as
\begin{equation}
\begin{aligned}\label{totalloss}
\mathcal{L}_{total}=\mathcal{L}_{known}+\lambda_{1}\mathcal{L}_{unknown}+\lambda_{2}\mathcal{L}_{center},
\end{aligned}
\end{equation}
where $\lambda_{1}$ and $\lambda_{2}$ are two trade-off parameters.
The contribution of training regime \eqref{totalloss} to OpenViewer's generalization is twofold: 
1) From a feature correspondence perspective, it ensures that known parts elicit a strong response, while undue confidences of pseudo-unknown are suppressed to a low response; 
2) From an entropy perspective, it reduces the entropy of known to augment discrimination, while increasing the entropy of pseudo-unknown to ensure that they have low confidence in being classified as known, thereby further reinforcing the recognition of unknown.
OpenViewer can be summarized as Algorithm \ref{algorithmOpenViewer} in \textbf{Appendix}.

\subsection{Main Theoretical Presentation and Analysis}\label{TheoreticalAnalysis}  
%First, we provide \textbf{Theorem \ref{Theoremshow2}}.

\begin{theorem}\label{Theoremshow2} \textbf{(Interpretability Boundary)}
If each sub-module is convergent, then the stacked deep unfolding network consisting of all modules is bounded. 
%(More details can be found in Appendix.)
\end{theorem}
\subsubsection{Remark 1.}
Supported by \textbf{Theorem \ref{Theoremshow2}}, the interpretable deep unfolding network \eqref{stotalnet2} will be bounded regardless of the initial multi-view cases with known and pseudo-unknown, indicating that information from different views can be reasonably integrated and \textbf{interpreted} in mixed scenarios, thereby improving the trustworthiness.
%\subsection{Theoretical Analysis: Overall Convergence}\label{TheoreticalAnalysis} 
%Then, we provide \textbf{Theorem \ref{Theoremshow5}}.
\begin{theorem}\label{Theoremshow5} \textbf{(Generalization Support)}
For the fixed step size (\emph{i.e.,} $\eta_{t} = \eta$) as $T \rightarrow \infty$, and given the existing upper boundary $\epsilon$, the difference $\mathcal{L}_T^*-\mathcal{L}^*$ generalizes to $\frac{\eta \epsilon^2}{2}$ with a convergence rate $\mathcal{O}(1/T)$.
%(More details can be found in Appendix.)
\end{theorem}
\subsubsection{Remark 2.} 
\textbf{Theorem \ref{Theoremshow5}} theoretically ensures that OpenViewer maintains stable \textbf{generalization} by learning a true distribution within a convergence radius of $\frac{\eta \epsilon^2}{2}$ and a convergence rate $\mathcal{O}(1/T)$, even when encountering unknown. 
%Therefore, the convergence rate of the proposed method is $\mathcal{O}(\frac{1}{T})$ and the convergence radius is $\frac{\eta \epsilon^2}{2}$ with the upper bound of positive real number $\epsilon$.
%\subsection{Complexity}
\subsubsection{Complexity.} 
The time complexity of OpenViewer with $L$ layers costs $\mathcal{O}\left(\left(NC(C+D_{v}+V)+C^{2}D_{v}\right)L\right)$, and the space complexity of OpenViewer denotes $\mathcal{O}\left(N(D_{v}+C)V\right)$.
Additional proofs and details of the main theories and complexity can be found in \textbf{Appendix}.

\section{Experiments and Studies}\label{Experiments}

%To evaluate the performance of OpenViewer, we conduct quantitative and qualitative experiments on eight real-world datasets. 
%Specifically, we choose to conduct experiments in challenging open-set classification tasks under multi-view environments. 
%The purpose of OpenViewer aims to address the following research issues:
%\begin{itemize}
%\item \textbf{RQ1:} Can OpenViewer achieve known and unknown recognition while addressing openness challenges?
%\item \textbf{RQ2:} Does each module and loss term really facilitate recognition? What is their level of contribution?
%\item \textbf{RQ3:} How do hyper-parameters affect performance and how to select the optimal value? 
%\end{itemize}

\subsubsection{Datasets, Compared Methods, and Evaluation Metric.} 
\begin{table}[t]
\centering
\small
\resizebox{0.45\textwidth}{!}{
\begin{tabular}{c||c|c|c|c}
\toprule
Datasets & \# Samples & \# Views & \# Feature Dimensions & \# Classes \\
\midrule
Animals & 10,158 & 2 & 4,096/4,096 & 50 \\
AWA & 30,475 & 6 & 2,688/2,000/252/2,000/2,000/2,000 & 50 \\
NUSWIDEOBJ  & 30,000 & 5 & 65/226/145/74/129 & 31 \\  
VGGFace2-50  & 34,027 & 4 & 944/576/512/640 & 50 \\
\midrule
ESP-Game & 11,032 & 2 & 100/100 & 7 \\
NUSWIDE20k & 20,000 & 2 & 100/100 & 8 \\
\bottomrule
\end{tabular}}
\caption{A brief description of the tested datasets.}
\label{Datadescription}
\end{table}

\begin{figure*}[t]
  \centering
  \includegraphics[width=0.95\textwidth]{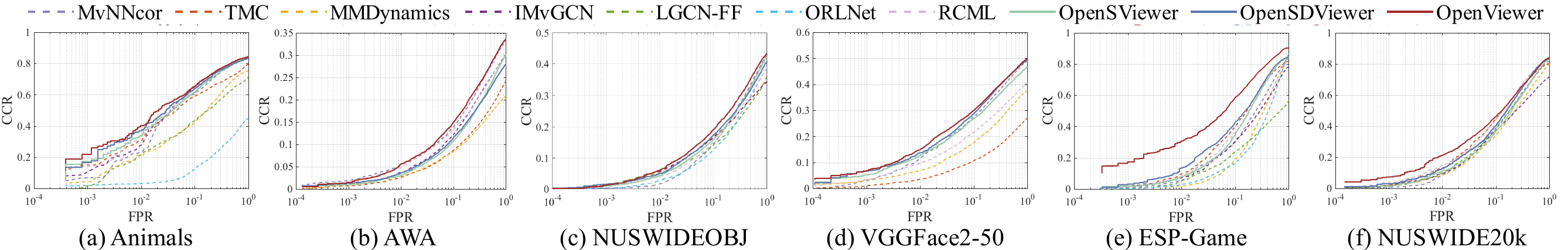}\\
  \caption{OSCR curves plotting the CCR over FPR on all test multi-view datasets for all compared methods.}
  \label{ccrfpr}
\end{figure*}

\begin{table*}[t]
\centering
\small
\resizebox{0.95\textwidth}{!}{
\begin{tabular}{cccccccccccccccc}
\toprule
Datasets & \multicolumn{5}{c}{Animals}  & \multicolumn{5}{c}{AWA}  & \multicolumn{5}{c}{NUSWIDEOBJ} \\  \cmidrule(r){2-6}   \cmidrule(r){7-11}  \cmidrule(r){12-16}  
Methods $\backslash$ CCR at FPR of & 0.5\% & 1.0\% & 5.0\% & 10\% & 50\% & 0.5\% & 1.0\% & 5.0\% & 10\% & 50\% & 0.5\% & 1.0\% & 5.0\% & 10\% & 50\% \\
\midrule
MvNNcor \cite{Xu2020Deepembedded} & 19.48 & 23.35 & 53.81 & 61.01 & 79.30 & 1.59 & 2.15 & 6.40 & 9.34 & 23.07 & 0.82 & 1.63 & 9.34 & 14.48 & 33.40 \\
TMC \cite{Han2021Trusted} & 24.30 & 30.36 & 51.22 & 59.25 & 72.04 & 1.67 & 2.60 & 6.09 & 8.68 & 18.05 & \textbf{4.17} & \underline{5.29} & 10.45 & 14.07 & 27.28 \\
MMDynamics \cite{Han2022Multimodal} & 16.93 & 21.51 & 35.65 & 42.39 & 69.11 & 1.78 & 2.75 & 7.04 & 9.86 & 20.19 & 3.39 & 5.10 & \underline{12.58} & \underline{17.72} & 30.59 \\
%PDMF \cite{Xu2023Progressive} (AAAI'23) &  &  &  &  &  &  &  &  &  &  &  &  &  &  &  &  &  &  \\
IMvGCN \cite{wu2023Interpretable} & 21.44 & 27.73 & 56.08 & 63.54 & 78.57 & 1.21 & 1.80 & 4.03 & 5.89 & 12.38 & 0.44 & 0.88 & 3.15 & 4.78 & 9.76 \\
LGCN-FF \cite{Chen2023Learnable} & 15.18 & 21.40 & 34.37 & 43.59 & 63.53 & - & - & - & - & - & - & - & - & - & - \\
ORLNet \cite{fang2024representation} & 10.76 & 12.46 & 23.08 & 31.71 & 49.81 & - & - & - & - & - & - & - & - & - & - \\
RCML \cite{Xu24Reliable} & 25.31 & 33.38 & 55.48 & 63.87 & \underline{80.03} & \textbf{3.52} & \underline{4.75} & \underline{10.11} & \underline{13.34} & \underline{24.69} & 2.99 & 4.26 & 10.44 & 15.22 & 30.19\\
\midrule
OpenSViewer & 27.64 & 33.83 & 53.50 & 61.75 & 79.54 & 2.34 & 3.09 & 7.55 & 10.77 & 23.30 & 3.02 & 4.37 & 11.41 & 16.08 & 32.64\\
OpenSDViewer & \underline{30.98} & \underline{37.05} & \underline{56.47} & \underline{64.21} & 79.61 & 2.36 & 3.56 & 8.09 & 11.49 & 23.54 & 3.60 & 5.02 & 12.47 & 17.43 & \underline{32.80} \\
OpenViewer & \textbf{31.56} & \textbf{40.24} & \textbf{58.07} & \textbf{65.38} & \textbf{80.78} & \underline{3.49} & \textbf{5.59} & \textbf{11.08} & \textbf{14.97} & \textbf{27.66} & \underline{4.12} & \textbf{6.16} & \textbf{14.17} & \textbf{19.12} & \textbf{35.19} \\
\midrule
\midrule
Datasets & \multicolumn{5}{c}{VGGFace2-50}  & \multicolumn{5}{c}{ESP-Game}  & \multicolumn{5}{c}{NUSWIDE20k} \\  \cmidrule(r){2-6}   \cmidrule(r){7-11}  \cmidrule(r){12-16}  
Methods $\backslash$ CCR at FPR of & 0.5\% & 1.0\% & 5.0\% & 10\% & 50\% & 0.5\% & 1.0\% & 5.0\% & 10\% & 50\% & 0.5\% & 1.0\% & 5.0\% & 10\% & 50\% \\
\midrule
MvNNcor \cite{Xu2020Deepembedded} & 8.29 & 11.65 & 22.31 & 27.53 & 43.01 & 3.48 & 5.72 & 21.48 & 30.15 & 68.22 & 2.90 & 8.24 & 26.59 & 38.69 & 70.62 \\
TMC \cite{Han2021Trusted} & 2.55 & 3.57 & 8.07 & 10.72 & 20.87 & 6.25 & 9.33 & 28.70 & 40.05 & 74.45 & 8.30 & 13.21 & \underline{33.04} & \underline{44.72} & 73.35 \\
MMDynamics \cite{Han2022Multimodal} & 5.55 & 6.91 & 13.70 & 17.62 & 31.53 & 0.99 & 1.52 & 8.92 & 18.71 & 64.20 & 5.67 & 8.77 & 24.63 & 33.88 & 66.16\\
%PDMF \cite{Xu2023Progressive} (AAAI'23) &  &  &  &  &  &  &  &  &  &  &  &  &  &  &  &  &  & \\
IMvGCN \cite{wu2023Interpretable} & - & - & - & - & - & 4.69 & 7.10 & 20.84 & 30.65 & 62.34 & 9.15 & 11.65 & 24.88 & 35.97 & 63.42\\
LGCN-FF \cite{Chen2023Learnable} & - & - & - & - & - & 2.18 & 5.07 & 15.56 & 23.84 & 45.17 & - & - & - & - & \\
ORLNet \cite{fang2024representation} & - & - & - & - & - & 1.65 & 2.61 & 11.26 & 16.67 & 56.48 & 6.10 & 9.85 & 27.08 & 37.55 & 70.72\\
RCML \cite{Xu24Reliable} & 7.85 & 9.85 & 17.21 & 21.64 & 34.60 & 4.62 & 7.77 & 24.31 & 36.25 & 75.91 & \underline{10.16} & \underline{14.94} & 31.61 & 44.29 & \underline{73.37} \\
\midrule
OpenSViewer & 10.18 & 12.77 & 21.80 & 26.86 & 40.82 & 4.66 & 8.16 & 21.94 & 34.91 & 75.24 & 6.49 & 11.26 & 25.55 & 36.39 & 70.66 \\
OpenSDViewer & \underline{10.75} & \underline{13.75} & \underline{23.23} & \underline{28.80} & \underline{43.28} & \underline{7.96} & \underline{13.35} & \underline{30.92} & \underline{42.56} & \underline{76.76} & 7.89 & 12.67 & 28.84 & 40.34 & 72.56\\
OpenViewer & \textbf{12.16} & \textbf{15.05} & \textbf{25.39} & \textbf{30.05} & \textbf{44.14} & \textbf{25.89} & \textbf{30.61} & \textbf{46.45} & \textbf{58.65} & \textbf{83.52} & \textbf{16.35} & \textbf{21.31} & \textbf{36.76} & \textbf{46.28} & \textbf{73.47}\\
\bottomrule
\end{tabular}}
\caption{CCR at different FPR is given for all comparison under \textbf{\emph{openness}} = 0.1 setting, where the best and runner-up results are highlighted in \textbf{boldface} and \underline{underlined}, respectively. ``–" indicates the out-of-memory error.}
\label{CCRClassification1}
\end{table*}

%\begin{figure*}[t]
%  \centering
%  \includegraphics[width=\linewidth]{./Figures/hist_all_multi_2.pdf}\\
%  \caption{Visualization of the confidence score of known and unknown samples of OpenViewer on all datasets.}
%  \label{histallmulti}
%\end{figure*}

\begin{figure*}[!htbp]
  \centering
  \includegraphics[width=0.95\linewidth]{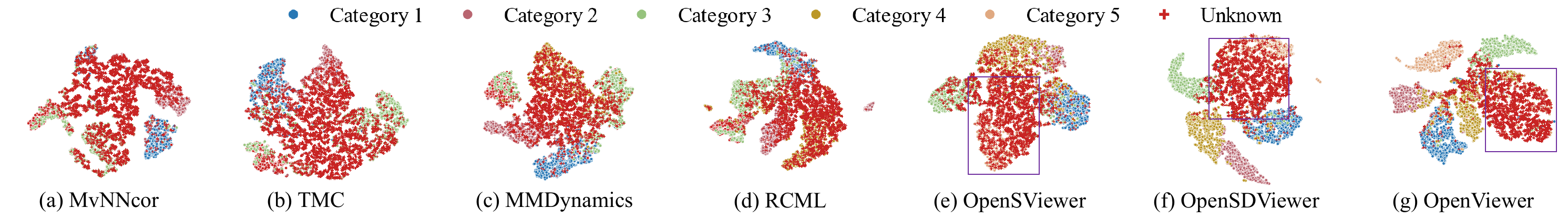}\\
  \caption{The t-SNE visualizations based on the fused representations of ESP-Game dataset.}
  \label{tnseESPGame1}
\end{figure*}

\begin{figure*}[!htbp]
  \centering
  \includegraphics[width=0.95\linewidth]{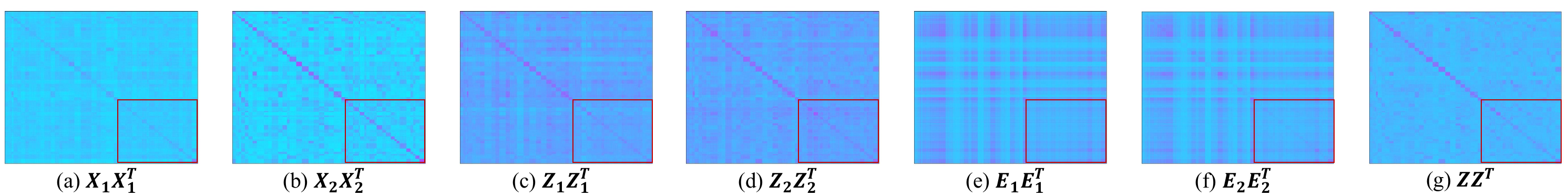}\\
  \caption{The heatmaps on feature, redundancy-free, noise and fusion matrices with $\mathbf{w}_1=0.5132$ and $\mathbf{w}_2=0.4868$ of Animals.}
  \label{heatmapAnimals}
\end{figure*}

\begin{figure}[!htbp]
  \centering
  \includegraphics[width=\linewidth]{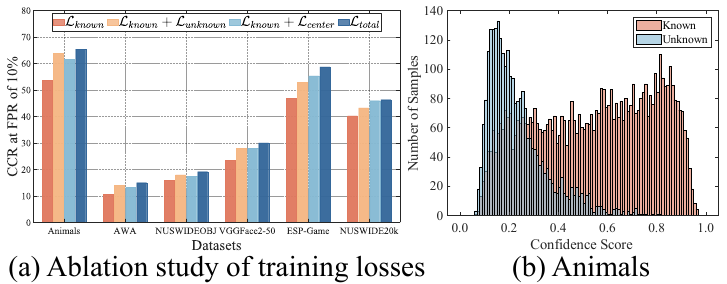}\\
  \caption{Ablation study of training losses on multi-view datasets with respect to CCR at FPR of 10\%.}
  \label{ablation}
\end{figure}

%\begin{figure*}[t]
%  \centering
%  \includegraphics[width=\linewidth]{./Figures/lambda_u_known_loss_all_multi.pdf}\\
%  \caption{Parameter sensitivity of $\lambda_{1}$ and $\lambda_{2}$ on VGGFace2-50 and NUSWIDE20k, $\mu$, $\boldsymbol{\xi}$, and loss behaviors of OpenViewer.}
%  \label{Parametersensitivityuknownloss}
%\end{figure*}

\begin{figure}[t]
  \centering
  \includegraphics[width=0.95\linewidth]{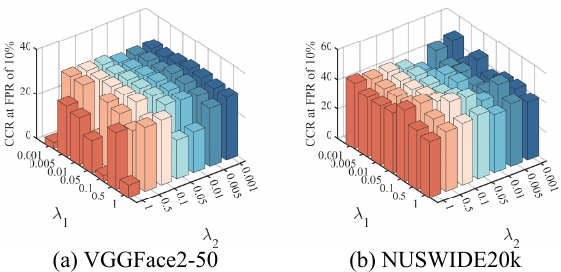}\\
  \caption{Parameter sensitivity of $\lambda_{1}$ and $\lambda_{2}$.}
  \label{Parametersensitivitylambda1lambda21}
\end{figure}

\begin{figure}[t]
  \centering
  \includegraphics[width=0.95\linewidth]{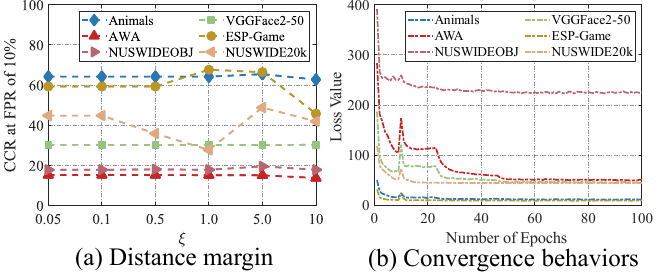}\\
  \caption{Parameter sensitivity of $\boldsymbol{\xi}$ and loss behaviors.}
  \label{Parametersensitivityuknownloss}
\end{figure}

We conduct experiments in challenging open-environment classification tasks under six well-known multi-view datasets. 
This includes two scenarios: 1) Animals, AWA, NUSWIDEOBJ, and VGGFace2-50 datasets contain different manual and deep features; 2) ESP-Game and NUSWIDE20k datasets include various vision and language features.
The statistics of these datasets are summarized in Table \ref{Datadescription} (details in \textbf{Appendix}).
%Image and text features among ESP-Game and NUSWIDE20k datasets are extracted via VGG-16 and BERT, respectively.
Moreover, to simulate the performance of OpenViewer in open-environment, we also utilize the concept of $\textbf{\emph{openness}}$ \cite{Scheirer2012Toward} to divide known and unknown categories of multi-view datasets.
%\begin{equation}\label{BiProblem}
%\begin{array}{ll}
%\textbf{\emph{openness}}=1-\sqrt{\frac{2\times C_{train}}{C_{train}+C_{test}}},
%\end{array}
%\end{equation}
%where $C_{train}$ is the number of known classes during training, and $C_{test}$ is the total number of known and unknown classes during testing.
%$\textbf{\emph{Openness}}$ describes the degree of openness of the dataset, which refers to the proportion of unknown classes in the training and testing sets. 
%Specifically, it measures the relative proportion of categories that a model has not known in the test set.
%In other words, for each dataset, we strategically designate a subset of classes as unknown classes for testing purposes, while utilizing the remaining classes as known classes for training. 
Meanwhile, the dataset is partitioned as follows: 10\% of the known class samples are allocated for training, another 10\% for validation, and the rest 80\% for testing.
%We want to keep the number of test samples high so that results are more stable.
%For each dataset, a portion of the classes are held out as the out-of-distribution class and used for testing, while the remaining classes are used as the intra-distribution classes for training. 
%The data is split such that 10$\%$ of the labeled samples are used for training, 10$\%$ for validation, and 80$\%$ for testing.
%The agent for identify the out-of-distribution class is determined using the validation set. 
%The number of out-of-distribution classes is varied to evaluate the performance of the models at different proportions of out-of-distribution classes.

Due to the limited exploration of related open multi-view learning tasks, we drew on backbone networks from other different multi-view tasks as compared methods (details in \textbf{Appendix}), including: MvNNcor \cite{Xu2020Deepembedded}, TMC \cite{Han2021Trusted}, MMDynamics \cite{Han2022Multimodal}, IMvGCN \cite{wu2023Interpretable}, LGCN-FF \cite{Chen2023Learnable}, ORLNet \cite{fang2024representation}, and RCML \cite{Xu24Reliable}. 
%For more information of compared methods, kindly see \textbf{Appendix} Subsection \ref{SupplementaryofExperiments}.
%These multi-view models are trained and tested according to their own structure.
%They utilize the same dataset settings as OpenViewer to test their performance in an open environment.

To estimate recognition performance effectively, the Open-Set Classification Rate (OSCR) \cite{Dhamija2018Reducing} is adopted as metrics, consisting of Correct Classification Rate (CCR) and False Positive Rate (FPR).
%By setting a probability threshold $\theta$, it can balance sensitivity to known classes with the ability to reject unknowns.
%For samples from known categories, we calculate the Correct Classification Rate (CCR) as the fraction of the samples where the correct class has maximum probability and has a probability greater than $\theta$. 
%Concurrently, False Positive Rate (FPR) is determined by the fraction of samples from the unknown category that are classified as any known class with a probability greater than $\theta$.
%Different applications may tolerate different levels of FPR for the benefit of higher CCR, making these specific metrics highly relevant for tuning the models according to specific needs. 
%For security-critical applications such as biometrics or fraud detection, a low FPR is prioritized over CCR to block unauthorized access.
%Conversely, content recommendation or advertising systems may allow a higher FPR to enhance inclusivity and user experience by ensuring relevant content is not overlooked.
%In order to further measure the quality of multi-view learning, we evaluate model performance with \textbf{CCR at FPR of 0.5\%/1.0\%/5.0\%/10.0\%/50.0\%}.
%Regarding the choice of thresholds, we follow the experimental setup in \cite{Dhamija2018Reducing} by taking the maximum probability set of predicted unknown samples and traversing it to compute the CCR and FPR.

\subsubsection{Experimental Setups.}\label{ExperimentalSetup}

OpenViewer is implemented using the PyTorch on an NVIDIA GeForce RTX 4080 GPU with 16GB of memory. 
We train OpenViewer for 100 epochs with a batch size of 50, a learning rate of 0.01, $\boldsymbol{\xi} = 5$, and $\lambda_{1}$ and $\lambda_{2}$ selected from $\{10^{-3}, 5\times10^{-3}, \cdots, 10^{0}\}$.
%Our approach employs a single-layer deep unfolding network, streamlining the model while maintaining its interpretability and efficiency.
The number of unfolding layers is set to $L = 1$ as suggested in \textbf{Appendix} Fig. \ref{Parametersensitivitylayers}, balancing complexity and efficiency while preserving interpretable expression-enhanced capabilities.
The ablation-models (\textbf{Appendix} Table \ref{NetworkExamples}) are OpenSViewer (w/o CD-Module and DN-Module) and OpenSDViewer (w/o DN-Module)  using for self-verification.

\subsubsection{Experimental Results.}\label{sec:exprq}
We present the overall OSCR curve results of all multi-view learning methods in classification under the condition of \textbf{\emph{openness}} = 0.1, as shown in Fig. \ref{ccrfpr}, from which we observe:
Intuitively, OpenViewer (red solid line) outperforms other methods (colored dashed lines) across all cases, whether on multi-feature or multi-modal datasets.
Although some multi-view methods, such as TMC and RCML, occasionally outperform OpenViewer in specific cases, it demonstrates more stable performance across all scenarios. 
Finally, OpenViewer's effective balance across different FPR and CCR results in outstanding performance.
On one hand, this may be attributed to the effective enhancement of both known and unknown expression through interpretable integration, such as the multi-feature expression enhancement seen with VGGFace2-50. 
On the other hand, the pseudo-unknown mechanism and perception-augmented loss contribute to significant confidence differentiation, as depicted in Fig. \ref{heatmapAnimals} (g) for Animals.
%Therefore, these reasons enable the model to reduce misclassifications of unknown samples as known, while maintaining accurate classification of known classes.

Meanwhile, we further extract representative \textbf{CCR at FPR of 0.5\%/1.0\%/5.0\%/10.0\%/50.0\%} and illustrate them in Table \ref{CCRClassification1}. 
It is evident that OpenViewer surpasses other methods in most cases, particularly under high FPR conditions (e.g., 50.0\%). 
Moreover, Fig. \ref{tnseESPGame1} (all in \textbf{Appendix} Fig. \ref{tnseESPGame}) indicates that OpenViewer achieves the highest separation between different categories and minimal overlap between unknown and known categories.
Additionally, Fig. \ref{heatmapAnimals} clearly demonstrates the impact of functionalized priors in OpenViewer. 
Original features exhibit complementary information but also contain redundancy, noise, and inappropriate confidences. 
Whereas OpenViewer effectively filters relevant features and noise, and suppresses inappropriate confidences for unknown (the red box). 
Ultimately, the weight-guided complementary fusion effectively enhances expression with a clean diagonalized structure, revealing high response for known and low response for unknown.
%These show OpenViewer's generalization and interpretability in an open environment, addressing key challenges.
%These all further demonstrate that the OpenViewer's interpretable integration and generalization to unknown data across multiple view, thereby addressing openness challenges and improving overall open-environment awareness.

\subsubsection{Ablation Study.}\label{sec:abla}
To verify that each module and loss term contributes to address openness challenges, we conduct ablation experiments.
First, when we comprehensively examine the variants OpenSViewer, OpenSDViewer, and OpenViewer portrayed in Fig. \ref{ccrfpr}, Fig. \ref{tnseESPGame1} (\textbf{Appendix} Fig. 1), and Table \ref{CCRClassification1}, we can excitedly discover that the addition of each interpretable module promotes performance improvement and effective separation of samples.
For example, after adding the diversity noise processing module (DN-Module), the performance of ESP-Game improved from 42.56\% to 58.65\%, with increased inter-class separability between known and unknown. 
This improvement can be attributed to the removal of multi-view noise, which enhances feature expression.
Furthermore, Fig. \ref{ablation} (a) progressively highlights the generalized contributions of all loss components. 
From this, their combination can promote the model's generalization by distinguishing between normal known and undue unknown confidences.
Taking the Animals dataset as an example, Fig. \ref{ablation} (b) depicts that training regime accentuates differences in confidence scores between known and unknown distributions, aiding in recognition. 
Specifically, the unknown loss enhances sample discrimination, while the center loss promotes intra-class compactness.

\subsubsection{Parameter Sensitivity Analysis.} \label{ParameterAnalysisRQ3}
First, Fig. \ref{Parametersensitivitylambda1lambda21} (all in \textbf{Appendix} Fig. \ref{Parametersensitivitylambda1lambda2}) illustrates the parameter sensitivity of OpenViewer on two representative datasets in terms of $\lambda_1$ and $\lambda_2$ of loss \eqref{totalloss}.
The model performance is generally robust in most cases, but it collapses when $\lambda_2$ becomes too large, causing all samples to provide meaningless confidences across all classes, with no clear winning class.
Second, Fig. \ref{Parametersensitivityuknownloss} (a) reveals when the parameter $\xi$ is set to 5, the feature range is adequate to effectively distinguish.
%This work precisely utilizes this value to strike a balance.
However, when this value is exceeded, the wider feature range causes overlap among known classes, leading to a decline in overall performance.
Finally, we showcase Fig. \ref{Parametersensitivityuknownloss} (b) to elucidate the behaviors between loss values and training epochs. 
The curve displays that after 100 training epochs, the loss value stabilizes, indicating convergence and underscoring its stability, as depicted in \textbf{Theoretical Analysis}.

\section{Conclusion and Future Work}\label{Conclusion}
In this paper, we proposed OpenViewer to address the openness challenges in open-settings.
OpenViewer began with a pseudo-unknown sample generation mechanism to previously adapt to unknown, followed by a multi-view expression-enhanced deep unfolding network to offer a more interpretable integration mechanism.
Additionally, OpenViewer employed a perception-augmented open-set training regime to improve generalization between known and unknown classes.
Extensive experiments on diverse multi-view datasets showed that OpenViewer outperformed existing methods in recognition while effectively tackling openness challenges.
In future work, we will explore more sophisticated openness-aware circumstances based multi-view learning, including heterogeneous or incomplete data. 

\bibliography{ML}

\clearpage

\appendix

\setcounter{section}{0}
\renewcommand\thesection{\Alph{section}}

%\section{Supplementary of Related Work}\label{SupRelatedWork}

%\subsection{Limitations, Solutions and Beyond}\label{subsec:Limitations} 

\section{Revisiting and Discussion}\label{Revisiting}

\setcounter{table}{0}
\setcounter{figure}{0}

\begin{algorithm}[!htbp]
\caption{OpenViewer}
\label{algorithmOpenViewer}
 \begin{algorithmic}[1]
\REQUIRE{Multi-view original known data $\mathcal{D}_{original}$, training epochs $T$, batch iterations $B = \lceil\frac{N}{Batchsize}\rceil$, the number of unfolding layers $L$, $\mathbf{Beta}$ distribution parameter $\omega$, trade-off parameters $\lambda_{1}, \lambda_{2}$, distance margin $\mathbf{\xi}$, and learning rate $\eta$}.
\ENSURE {Obtain predictive distribution $P(\mathbf{Z}_{test})$.}
\STATE {Initialize network parameters $\mathbf{\Theta}=\{\mathbf{R}, \mathbf{U}, \mathbf{M}, \theta_{v}, \rho_{v}\}$;}
\STATE {Initialize centroids $\{\boldsymbol{o}_{v}\}_{v=1}^{V}$ and center vectors $\{\boldsymbol{c}^{\hat{\boldsymbol{y}}_{c}}\}_{c=1}^{C}$;}
\FOR {$t = 1 \rightarrow T$}
\FOR {$b = 1 \rightarrow B$}
\STATE {Generate the pseudo-unknown data set $\mathcal{D}_{generated}$ by Eq. (1);}
\STATE {Concatenate $\mathcal{D}_{original}$ and $\mathcal{D}_{generated}$ to $\mathcal{D}_{train}$;}
\FOR {$l = 1 \rightarrow L$}
\STATE {Calculate and obtain the complementarity enhanced representation $\mathbf{Z}$ by network (8);}
\ENDFOR
\STATE {Compute the loss by open-set training regime (13);}
\STATE {Update $\mathbf{\Theta}$ though backward propagation;}
\ENDFOR
\ENDFOR
\RETURN{Obtain predictive distribution $P(\mathbf{Z}_{test})$.}
\end{algorithmic}
\end{algorithm}

\begin{table*}[t]
\centering
\resizebox{\textwidth}{!}{
\begin{tabular}{c|c|l}
\toprule
\multicolumn{1}{c|}{Variant Models} & \multicolumn{1}{c|}{Concretized Optimization Problems}  & \multicolumn{1}{c}{Composition of Deep Unfolding Network Modules}\\
%\midrule
%OpenViewer & \textbf{Generalized Problem \eqref{q3}} & \textbf{Interpretable Deep Unfolding Network Architecture \eqref{stotalnet1}} \\
\midrule
\multirow{2}{*}{OpenSViewer} & \multirow{2}{*}{$\mathop{\min}\limits_{\mathbf{Z}_{v}}\sum\limits_{{v=1}}^{V}(\frac{1}{2}\Vert\mathbf{X}_{v}-\mathbf{Z}_{v}\mathbf{D}_{v}\Vert_{F}^{2}+\alpha \Vert\mathbf{Z}_{v}\Vert_{1})$} & 
\textbf{RF-Module:} $\mathbf{Z}_{v}^{(l+1)}\leftarrow\mathbfcal{S}_{\theta_{v}^{(l)}}\left(\mathbf{Z}_{v}^{(l)}\mathbf{R}+\mathbf{X}_{v}\mathbf{U}\right)$\\
& & \textbf{CW-Fusion:} $\mathbf{Z}^{(l+1)}\leftarrow\sum\limits_{v=1}^V \mathbf{w}_{v}^{(l+1)} \mathbf{Z}_v^{(l+1)}$\\
\midrule
\multirow{3}{*}{OpenSDViewer} & \multirow{3}{*}{$\mathop{\min}\limits_{\mathbf{Z}_{v}, \mathbf{D}_{v}}\sum\limits_{{v=1}}^{V}(\frac{1}{2}\Vert\mathbf{X}_{v}-\mathbf{Z}_{v}\mathbf{D}_{v}\Vert_{F}^{2}+\alpha \Vert\mathbf{Z}_{v}\Vert_{1}+\frac{\beta}{2}\Vert\mathbf{D}_{v}\Vert_{F}^{2})$} &
\textbf{RF-Module:} $\mathbf{Z}_{v}^{(l+1)}\leftarrow\mathbfcal{S}_{\theta_{v}^{(l)}}\left(\mathbf{Z}_{v}^{(l)}\mathbf{R}+\mathbf{X}_{v}(\mathbf{D}_{v}^\top)^{(l)}\mathbf{U}\right)$ \\
& & \textbf{CD-Module:} $\mathbf{D}_{v}^{(l+1)}\leftarrow\mathbf{M}(\mathbf{Z}_{v}^\top)^{(l+1)}\mathbf{X}_{v}$ \\
& & \textbf{CW-Fusion:} $\mathbf{Z}^{(l+1)}\leftarrow\sum\limits_{v=1}^V \mathbf{w}_{v}^{(l+1)} \mathbf{Z}_v^{(l+1)}$\\
\midrule
\multirow{4}{*}{OpenViewer} & \multirow{4}*{$\mathop{\min}\limits_{\mathbf{Z}_{v}, \mathbf{D}_{v},\mathbf{E}_{v}}\sum\limits_{{v=1}}^{V}(\frac{1}{2}\Vert\mathbf{X}_{v}-\mathbf{Z}_{v}\mathbf{D}_{v}-\mathbf{E}_{v}\Vert_{F}^{2}+\alpha \Vert\mathbf{Z}_{v}\Vert_{1} +\frac{\beta}{2}\Vert\mathbf{D}_{v}\Vert_{F}^{2}+\gamma \Vert\mathbf{E}_{v}\Vert_{2,1})$} &
\textbf{RF-Module:} $\mathbf{Z}_{v}^{(l+1)}\leftarrow\mathbfcal{S}_{\theta_{v}^{(l)}}\left(\mathbf{Z}_{v}^{(l)}\mathbf{R}+(\mathbf{X}_{v}-\mathbf{E}_{v}^{(l)})(\mathbf{D}_{v}^\top)^{(l)}\mathbf{U}\right)$ \\
& & \textbf{CD-Module:} $\mathbf{D}_{v}^{(l+1)}\leftarrow\mathbf{M}(\mathbf{Z}_{v}^\top)^{(l+1)}(\mathbf{X}_{v}-\mathbf{E}_{v}^{(l)})$ \\
& & \textbf{DN-Module:} $\mathbf{E}_{v}^{(l+1)}\leftarrow\mathbfcal{P}_{\rho_{v}^{(l)}}\left(\mathbf{X}_{v}- \mathbf{Z}_{v}^{(l+1)}\mathbf{D}_{v}^{(l+1)}\right)$\\
& & \textbf{CW-Fusion:} $\mathbf{Z}^{(l+1)}\leftarrow\sum\limits_{v=1}^V \mathbf{w}_{v}^{(l+1)} \mathbf{Z}_v^{(l+1)}$\\
\bottomrule
\end{tabular}}
\caption{The generalized objective problem (2) leverages the ADMM optimizer to obtain the deep unfolding network (8). 
This inspires various multi-view ablation deep unfolding networks with diverse modules combination.}
%\vspace{2pt}
\label{NetworkExamples}
\end{table*}

\subsubsection{Revisiting.} To address the openness challenges in multi-view learning within open environments, we propose a expression-enhanced deep unfolding network and a perception-augmented open-set training regime, designed to manage both known and unknown instances. 
OpenViewer decouples the multi-view integration process, enabling the model to effectively distinguish between known and unknown, thus mitigating issues of interpretability and generalization. 
OpenViewer promotes the expression of known and unknown through the unfolding network, ensuring that known parts are highly responsive while unknown parts highlight inappropriate confidences. 
Additionally, the loss assistance model further strengthens the confidences in known while suppressing undue confidences in unknown.
OpenViewer can be summarized as Algorithm \ref{algorithmOpenViewer}.
\subsubsection{Discussion.} 1) \textbf{Connections with Existing Multi-view Methods:} OpenViewer combines the prior knowledge and architecture of heuristic \cite{Wang22Alignthen, Liu24Samplelevel} and deep \cite{Lin23Dualcontrastive, Du2024UMCGL} methods to adapt to unknown class samples while ensuring interpretability;
2) \textbf{Differences with Existing Interpretable and Open-set Methods:} Unlike existing interpretable methods \cite{wu2023Interpretable, Weerdt24DeepUnfolding, fang2024representation}, OpenViewer decouples model interpretability in mixed scenarios through multi-view priors, intuitively reflecting the essence and function of data operations during the integration process in openness multi-view learning.
Furthermore, OpenViewer distinguishes itself from \cite{Duan23Graphanomaly, Gou24Test}, as it is one of the few open-set approaches that simultaneously ensures the perception of both known and unknown instances in multi-view scenarios while also addressing interpretability.

\section{Supplementary of Interpretability Boundary}\label{SupplementaryofNetworkModuleConvergence}

Certainly, the construction of network modules in OpenViewer adheres to the characteristics commonly associated with heuristic methods. 
As a result, by employing the conventional approach of heuristic methods and updating each sub-variable problem using ADMM, we can theoretically ensure the interpretability boundary of OpenViewer.
Moreover, if we establish the interpretability boundary of each individual module, it naturally follows that OpenViewer, comprising these modules, will also converge and be bounded. 
Embracing these principles will guide us in proving and analyzing the interpretability boundary of OpenViewer.
For the convenience of analysis, we can divide these modules into two classes: RF-Module, CD-Module, and DN-Module with parameter updating, and CW-Fusion modules only need self-optimization without parameters.
Then, the boundary will also be analyzed according to this classification. 
If they are all convergent, then as the previous analysis, their overall combination can also ensure boundness.

\textbf{CW-Fusion:} In theory, this module converges due to the stability of weight calculation and the characteristics of linear combination.

\textbf{RF-Module, CD-Module, and DN-Module:} Unlike the above modules, these are iterative network modules with parameters and cannot be subjected to analyze as the above modules.
Fortunately, several works \cite{Liu21EIGNN, Liu2022MGNNI} have provided us with ideas for analyzing the boundness of such modules. 
Following them, we first need to introduce \textbf{Definition \ref{Definition5}} of the Banach fixed Point Theorem \cite{Shukla16Generalized} as a preparation. 

\begin{definition}\label{Definition5} 
Let $(\mathbf{X}, d)$ be a non-empty complete metric space with a contraction mapping $T: \mathbf{X} \rightarrow \mathbf{X}$.
If there exists a $\mu \in [0, 1)$ that makes $d(T(\mathbf{x}), T(\mathbf{y})) \leq \mu d(\mathbf{x}, \mathbf{y}), \forall \mathbf{x}, \mathbf{y} \in \mathbf{X}$.  
Then, a contraction $T$ admits a unique fixed-point $\mathbf{x}^{*}$ in $\mathbf{X}$ (\emph{i.e.,} $T(\mathbf{x}^{*})= \mathbf{x}^{*}$).
Furthermore, $\mathbf{x}^{*}$ can be found as follows: Start with an arbitrary element $\mathbf{x}_{0} \in \mathbf{X}$ and and define a sequence $\{\mathbf{x}_{n}\}_{n \in \mathbb{N}}$ by $\mathbf{x}_{n} = T\{\mathbf{x}_{n-1}\}$ for $n \geq 1$. 
Then $\lim_{n \rightarrow \infty} \mathbf{x}_{n}=\mathbf{x}^{*}$.
\end{definition} 

With the help of \textbf{Definition \ref{Definition5}}, we can transit to \textbf{Theorem \ref{Theorem1}}.

\begin{theorem}\label{Theorem1}
Given the bounded damping factor $\mu \in [0, 1)$, the modules for propagation is a contraction mapping, and the unique convergence solution $\mathbf{X}^{*}$ can be obtained.
\end{theorem}

\begin{proof}\label{Proof1}
For any matrix $\mathbf{B} \in \mathbb{R}^{B_1 \times B_2}$, we define the vectorization of the matrix by $vec[\mathbf{B}]\in\mathbb{R}^{B_1 \times B_2}$ and the Frobenius norm of the matrix by $\|\mathbf{B}\|_F$.
Let us first state some preliminary explanation for the framework derived networks in Table \ref{NetworkExamples}.
On this premise, we define the mapping function $\varphi$ that contains the proximal operator $\mathcal{S}_{\theta_{v}^{(l)}}$.
Then, $\varphi(\mathbf{Z}_v) = \mathbf{Z}_{v}\mathbf{R}+(\mathbf{X}_{v}-\mathbf{E}_{v})\mathbf{D}_{v}^\top\mathbf{U}$, and we want to present that the map $\varphi$ is contraction.
Using the property of the vectorization $vec$ and the Kronecker product $\otimes$,
\begin{equation}\label{ProblemDerivation}
\begin{aligned}
&vec[\varphi(\mathbf{Z}_v)]=vec\left[\mathbf{Z}_v\mathbf{R}\right]+vec[(\mathbf{X}_{v}-\mathbf{E}_{v})\mathbf{D}_{v}^\top\mathbf{U}]\\&=\left[\left(\mathbf{R}\right)^{\top}\right]vec[\mathbf{Z}_v]+vec[(\mathbf{X}_{v} -\mathbf{E}_{v})\mathbf{D}_{v}^\top\mathbf{U}]\\&=\mu\left(\left[\left(\mathbf{F}\right)^{\top}\otimes \mathbf{A}\right]\right) vec[\mathbf{Z}_v]+vec[(\mathbf{X}_{v} -\mathbf{E}_{v})\mathbf{D}_{v}^\top\mathbf{U}],
\end{aligned}
\end{equation}
where $\mu$ is an extracted parameter, and $\mathbf{A} \in \mathbb{R}^{N \times N}$ is a matrix whose diagonal element is $\frac{1}{\mu}$ and the rest is 0.
Therefore, for any $\mathbf{Z}_v, \mathbf{Z}_v^{\prime} \in \mathbb{R}^{N \times C}$,
\begin{equation}\label{ProblemDerivation2}
\begin{aligned}
&\left\|\varphi(\mathbf{Z}_v)-\varphi\left(\mathbf{Z}_v^{\prime}\right)\right\|_F  =\left\|vec[\varphi(\mathbf{Z}_v)]-vec\left[\varphi\left(\mathbf{Z}_v^{\prime}\right)\right]\right\|_2 \\&=\left\|\mu\left(\left[\left(\mathbf{F}\right)^{\top}\otimes \mathbf{A}\right]\right)\left(vec[\mathbf{Z}_v]-vec\left[\mathbf{Z}_v^{\prime}\right]\right)\right\|_2 \\&\leq \mu\left\|\left(\left[\left(\mathbf{F}\right)^{\top}\otimes \mathbf{A}\right]\right)\right\|_2\left\|vec[\mathbf{Z}_v]-vec\left[\mathbf{Z}_v^{\prime}\right]\right\|_2.
\end{aligned}
\end{equation}
Due to all matrices and layers are normalized, which means the spectral radius of $A = \left[\left(\mathbf{F}\right)^{\top}\otimes \mathbf{A}\right]\in [0, \frac{1}{\mu}-\frac{1}{\mu L_{p_{v}}}]$, and the range of $\left\|A\right\|_2$ belongs to $[0, 1)$,
\begin{equation}\label{ProblemDerivation3}
\begin{aligned}
&\left\|\varphi(\mathbf{Z}_v)-\varphi\left(\mathbf{Z}_v^{\prime}\right)\right\|_F \leq \mu\underbrace{\left\|\left(\left[\left(\mathbf{F}\right)^{\top}\otimes \mathbf{A}\right]\right)\right\|_2}_{[0, 1)}\\&\left\|vec[\mathbf{Z}_v]-vec\left[\mathbf{Z}_v^{\prime}\right]\right\|_2
\leq \mu\left\|\mathbf{Z}_v-\mathbf{Z}_v^{\prime}\right\|_F.
\end{aligned}
\end{equation}
Since $\mu \in [0, 1)$, this illustrates that $\varphi$ is a contraction mapping on the metric space ($\mathbb{R}^{N \times C}, \hat{d}$), where $\hat{d}\left(\mathbf{Z}_v, \mathbf{Z}_v^{\prime}\right)=\left\|\mathbf{Z}_v-\mathbf{Z}_v^{\prime}\right\|_F$.
Therefore, this module has a convergence lower bound, and the proof is completed.
\end{proof}

Similarly, we can also prove the boundness of other modules as above.
By leveraging the properties of matrix vectorization, the Kronecker product, and \textbf{Definition \ref{Definition5}}, the boundness of these interpretable modules can be demonstrated by \textbf{Theorem \ref{Theorem1}}. 
By ensuring the boundness of each of the four modules, we can establish the following \textbf{Theorem \ref{Theorem2}}.

\begin{theorem}\label{Theorem2} \textbf{(Interpretability Boundary)}
If each sub-module is convergent, then the stacked deep unfolding network consisting of all modules is bounded. 
\end{theorem}

\begin{proof}\label{Proof2}
Since each sub-module gradually approaches its own optimal value, their joint deep unfolding network obtain the bounded solution of a large global problem by coordinating the sub-module solutions until convergence.
\end{proof}

\textbf{Theorem \ref{Theorem2}} ensures that the deep network exported through the ADMM method will be bounded to the same solution regardless of which starting point the iteration starts from.
This indicates that even in the face of different views, through interpretable operations, it is possible to ensure that information from these different views can be reasonably integrated, thereby improving the trustworthiness of feature expression enhancement.
It also ensures that the network has interpretable theoretical guarantees.

%Each network module ensures convergence, and parameters are updated by backward propagation along with losses, thus ensuring the convergence of the entire deep unfolding networks.
At this point, we have ensured the interpretability boundness of unfolding networks through \textbf{Theorem \ref{Theorem2}}.

\section{Supplementary of Generalization Support}\label{SupplementaryofOverallConvergence}

Here, \cite{Peng2022XAI} provides us with some ideas to prove the overall convergence.
First, we use the standard gradient descent methods to update parameters $\mathbf{\Theta}$ as $\mathbf{\mathbf{\Theta}}_{t+1}=\mathbf{\Theta}_{t}-\eta_{t} \nabla \mathcal{L}_{total}\left(\mathbf{\Theta}_{t}\right)$, where $\eta_{t}$ is a learning rate, $\nabla \mathcal{L}_{total}$ denotes the gradient of total losses, and we need to introduce \textbf{Definition \ref{Definition6}} as follows.
\begin{definition}\label{Definition6}
A function $f(x)$ is Lipschitz continuous on the set $\Phi$, if there exists a constant $\epsilon > 0$, $\forall x_{1}, x_{2} \in \Phi$ such that $\left\|f\left(x_1\right)-f\left(x_2\right)\right\| \leq \epsilon\left\|x_1-x_2\right\|$, where $\epsilon$ is the Lipschitz constant.  
\end{definition}

Namely, the objective function $\mathcal{L}_{total}(\mathbf{\Theta})$ is Lipschitz continuous \textit{i.i.f.} $\left\|\mathcal{L}_{total}\right\|\leq \epsilon$. 
In other words, to meet the Lipschitz continuity, we need to prove that the upper boundary of $\nabla\mathcal{L}_{total}$ exists. 
To this end, we propose \textbf{Theorem \ref{Theorem3}}.
\begin{theorem}\label{Theorem3}
There exist $\epsilon > 0$ such as $\nabla\mathcal{L}_{total} \leq \epsilon$, where $\epsilon$ can be $\frac{1}{N^{o}}\left\|\hat{\boldsymbol{p}}_{c}^{(i)}\right\|  + \frac{1}{N^{o}}\left\|\hat{\boldsymbol{y}}_{c}^{(i)}\right\| + 2\left\|\hat{\boldsymbol{z}}^{(i)}\right\|+2\boldsymbol{\xi}+\lambda_1(\frac{1}{C}\left\|\tilde{\boldsymbol{p}}^{(i)}\right\|+\frac{1}{C}+2\left\|\tilde{\boldsymbol{z}}^{(i)}\right\|)+\lambda_2(\left\|\hat{\boldsymbol{z}}_{c}^{(i)}\right\|+ \left\|\boldsymbol{c}^{\hat{\boldsymbol{y}}_{c}^{(i)}}\right\|+\varphi)$ at least.
\end{theorem}
%\textbf{Proof.} Please refer to \textbf{Proof of Theorem 2} of Section \ref{Proof} of the \textbf{Appendix}.

\begin{proof}\label{Proof3}
The derivation of $\mathcal{L}_{known}$, $\mathcal{L}_{unknown}$, and $\mathcal{L}_{center}$ are performed separately as
\begin{equation}\label{lknownproof}
\begin{aligned}
\nabla \mathcal{L}_{known}(\hat{\boldsymbol{z}}^{(i)}) = \frac{1}{N^{o}} \left(\hat{\boldsymbol{p}}_{c}^{(i)}-\hat{\boldsymbol{y}}_{c}^{(i)} \right) + 
2 (\|\hat{\boldsymbol{z}}^{(i)}\| - \boldsymbol{\xi}) \cdot \frac{\hat{\boldsymbol{z}}^{(i)}}{\|\hat{\boldsymbol{z}}^{(i)}\|},
\end{aligned}
\end{equation}
\begin{equation}\label{lunknownproof}
\nabla\mathcal{L}_{unknown}(\tilde{\boldsymbol{z}}^{(i)}) = \frac{1}{C} \left(\tilde{\boldsymbol{p}}^{(i)} - 1\right) + 2 \tilde{\boldsymbol{z}}^{(i)},
\end{equation}
\begin{equation}\label{lcenterproof}
\begin{aligned}
\nabla \mathcal{L}_{center}(\hat{\boldsymbol{z}}_{c}^{(i)}) &= (\hat{\boldsymbol{z}}_{c}^{(i)} - \boldsymbol{c}^{\hat{\boldsymbol{y}}_{c}^{(i)}}) + \lambda \frac{\delta\left(\hat{\boldsymbol{y}}_{c}^{(i)}=j\right)}{1 + \sum_{i=1}^{N^{o}} \delta\left(\hat{\boldsymbol{y}}_{c}^{(i)}=j\right)},
\end{aligned}
\end{equation}
where $\hat{\boldsymbol{p}}_{c}^{(i)} = P(\hat{\boldsymbol{z}}^{(i)}) = \frac{e^{\hat{\boldsymbol{z}}^{(i)}}}{\sum\limits_{j=1}^{C} e^{\hat{\boldsymbol{z}}_{j}^{(i)}}}$, $\tilde{\boldsymbol{p}}^{(i)} = P(\hat{\boldsymbol{z}}^{(i)}) $.
Then, according to the properties of matrices and their inequalities, the following inequalities hold,
\begin{equation}\label{inlknownproof}
\begin{aligned}
\left\|\nabla \mathcal{L}_{known}(\hat{\boldsymbol{z}}^{(i)})\right\| \leq \frac{1}{N^{o}}(\left\|\hat{\boldsymbol{p}}_{c}^{(i)}\right\|+\left\|\hat{\boldsymbol{y}}_{c}^{(i)}\right\|)+2(\left\|\hat{\boldsymbol{z}}^{(i)}\right\| + \boldsymbol{\xi}), 
\end{aligned}
\end{equation}
\begin{equation}\label{inlunknownproof}
\left\|\nabla\mathcal{L}_{unknown}(\tilde{\boldsymbol{z}}^{(i)}) \right\| \leq \frac{1}{C} \left(\left\|\tilde{\boldsymbol{p}}^{(i)}\right\| + 1\right) + 2 \left\|\tilde{\boldsymbol{z}}^{(i)}\right\|,
\end{equation}
\begin{equation}\label{inlcenterproof}
\left\|\nabla \mathcal{L}_{center}(\hat{\boldsymbol{z}}_{c}^{(i)})\right\| \leq \left\|\hat{\boldsymbol{z}}_{c}^{(i)}\right\| + \left\|\boldsymbol{c}^{\hat{\boldsymbol{y}}_{c}^{(i)}}\right\|+\varphi,
\end{equation}
where $\varphi = \lambda \frac{\delta\left(\hat{\boldsymbol{y}}_{c}^{(i)}=j\right)}{1 + \sum_{i=1}^{N^{o}} \delta\left(\hat{\boldsymbol{y}}_{c}^{(i)}=j\right)}$, $\lambda$ is a constant, $0 \leq \left\|\hat{\boldsymbol{p}}\right\| \leq 1$, $0 \leq \left\|\tilde{\boldsymbol{p}}\right\| \leq 1$, $0 \leq \left\| \hat{\boldsymbol{z}} \right\| \leq 1$, and $0 \leq \left\| \tilde{\boldsymbol{z}} \right\| \leq 1$.
Finally, based on all the above statements, we can obtain the following boundary inequalities as
\begin{equation}\label{intotalloss}
\begin{aligned}
\left\|\nabla \mathcal{L}_{total}(\boldsymbol{z})\right\| &\leq \left\|\nabla\mathcal{L}_{known}(\hat{\boldsymbol{z}})\right\|+\lambda_1\left\|\nabla\mathcal{L}_{unknown}(\tilde{\boldsymbol{z}})\right\|\\&+\lambda_2\left\|\nabla\mathcal{L}_{center}(\hat{\boldsymbol{z}})\right\|,
\end{aligned}
\end{equation}
where $\boldsymbol{z} = \{\hat{\boldsymbol{z}}, \tilde{\boldsymbol{z}}\}$.
It showcases that the training loss function $\mathcal{L}_{total}(\mathbf{\Theta})$ will be upper bounded by a positive real number $\epsilon$ when $\left\|\boldsymbol{z}\right\|$ is bounded. 
At this point, the proof is completed.
In fact, there exists the upper boundary of $\left\|\boldsymbol{z}\right\|$ for any real-world datasets.
\end{proof}

Based on \textbf{Theorem \ref{Theorem3}}, we have the following convergence \textbf{Theorem \ref{Theorem4}}.
$\mathcal{L}_{total}(\mathbf{\Theta})$ is abbreviated as $\mathcal{L}$.
\begin{theorem}\label{Theorem4}
One could always find an optimal loss $\mathcal{L}_{T}^{*}(\mathbf{\Theta})$, which is sufficiently close to the desired $\mathcal{L}^{*}(\mathbf{\Theta})$ after $T$ epochs, \emph{i.e.,} $\mathcal{L}_T^*-\mathcal{L}^* \leq \frac{\left\|\mathbf{\Theta}_1-\mathbf{\Theta}^*\right\|_F^2+\epsilon^2 \sum\limits_{t=1}^T (\eta_{t})^2}{2 \sum\limits_{t=1}^T \eta_{t}}$.
\end{theorem}

%\textbf{Proof.} Please refer to \textbf{Proof of Theorem 3} of Section \ref{Proof} of the \textbf{Appendix}.

\begin{proof}\label{Proof4}
Let $\mathbf{\Theta}^*$ be the optimal parameters of OpenViewer, $\mathbf{\Theta}_{T}$ be the $T$-th parameters, and then we have
\begin{equation}\label{theorem3pro1}
\begin{aligned}
&~~~~~~~~~~\left\|\mathbf{\Theta}_{T+1}-\mathbf{\Theta}^*\right\|_F^2=\left\|\mathbf{\Theta}_T-\mathbf{\Theta}^*\right\|_F^2 \\&-2 \operatorname{tr}\left(\eta_T \nabla \mathcal{L}_T^{\top}\left(\mathbf{\Theta}_T-\mathbf{\Theta}^*\right)\right)+\eta_t^2\left\|\nabla \mathcal{L}_T\right\|_F^2.
\end{aligned}
\end{equation}
By recursively applying the above equation, we obtain
\begin{equation}\label{theorem3pro2}
\begin{aligned}
&~~~~~~~~~~\left\|\mathbf{\Theta}_{T+1}-\mathbf{\Theta}^*\right\|_F^2=\left\|\mathbf{\Theta}_1-\mathbf{\Theta}^*\right\|_F^2\\&-2 \sum_{t=1}^T \eta_t t r\left(\mathbf{\Theta}_t-\mathbf{\Theta}^*\right)+\sum_{t=1}^T \eta_t^2\left\|\nabla \mathcal{L}_t\right\|_F^2.
\end{aligned}
\end{equation}
As $\mathcal{L}(\cdot)$ satisfies the Lipschitz Continuity and according to the definition of gradient, \emph{i.e.,}
\begin{equation}\label{theorem3pro3}
\begin{aligned}
f\left(x^*\right) \geq f\left(x_t\right)+\nabla \mathcal{L}_t^{\top}\left(x^*-x_t\right),
\end{aligned}
\end{equation}
and then, we have
\begin{equation}\label{theorem3pro4}
\begin{aligned}
&~~\left\|\mathbf{\Theta}_{T+1}-\mathbf{\Theta}^*\right\|_F^2 \leq\left\|\mathbf{\Theta}_{\mathbf{1}}-\mathbf{\Theta}^*\right\|_F^2\\&-2 \sum_{t=1}^T \eta_t\left(\mathcal{L}_t-\mathcal{L}^*\right)+\epsilon^2 \sum_{t=1}^T \eta_t^2,
\end{aligned}
\end{equation}
\begin{equation}\label{theorem3pro5}
\begin{aligned}
2 \sum_{t=1}^T \eta_t\left(\mathcal{L}_t-\mathcal{L}^*\right) \leq\left\|\mathbf{\Theta}_1-\mathbf{\Theta}^*\right\|_F^2+\epsilon^2 \sum_{t=1}^T \eta_t^2,
\end{aligned}
\end{equation}
\begin{equation}\label{theorem3pro6}
\begin{aligned}
\mathcal{L}_t-\mathcal{L}^* \geq \min _{t=1,2, \cdots, T}\left(\mathcal{L}_t-\mathcal{L}^*\right)=\mathcal{L}_T^*-\mathcal{L}^*,
\end{aligned}
\end{equation}
where $\mathcal{L}_T^*$ is the best $\mathcal{L}$ found in $T$ epochs.
Combining inequalities \eqref{theorem3pro5}-\eqref{theorem3pro6}, we finally have
\begin{equation}\label{theorem3pro7}
\begin{aligned}
\mathcal{L}_T^*-\mathcal{L}^* \leq \frac{\left\|\mathbf{\Theta}_1-\mathbf{\Theta}^*\right\|_F^2+\epsilon^2 \sum\limits_{t=1}^T \eta_t^2}{2 \sum\limits_{t=1}^T \eta_t}.
\end{aligned}
\end{equation}
Here, the proof is completed.
\end{proof}

Grounded on \textbf{Theorem \ref{Theorem4}}, we could derive \textbf{Theorem \ref{Theorem5}}.
\begin{theorem}\label{Theorem5} \textbf{(Generalization Support)}
For the fixed step size (\emph{i.e.,} $\eta_{t} = \eta$) as $T \rightarrow \infty$, and given the existing upper bound of positive real number $\epsilon$, the difference $\mathcal{L}_T^*-\mathcal{L}^*$ generalizes to $\frac{\eta \epsilon^2}{2}$ with a convergence rate $\mathcal{O}(\frac{1}{T})$.
\end{theorem}

\begin{proof}\label{Proof5}
After $T$ epochs, we have
\begin{equation}\label{lemma1pro1}
\begin{aligned}
\mathcal{L}_T^*-\mathcal{L}^* & \leq \frac{\left\|\mathbf{\Theta}_1-\mathbf{\Theta}^*\right\|_F^2+T \epsilon^2 \eta^2}{2 T \eta}\\&=\frac{\left\|\mathbf{\Theta}_1-\mathbf{\Theta}^*\right\|_F^2 /(T \eta)+\eta \epsilon^2}{2}.
\end{aligned}
\end{equation}
The first term decreases at the rate of $\frac{1}{T}$, while the second term is a constant term for a fixed learning rate $\eta$.
Therefore, the convergence rate of the proposed method is $\mathcal{O}(\frac{1}{T})$ and the convergence radius is $\frac{\eta \epsilon^2}{2}$ with the upper bound of positive real number $\epsilon$.
Here, the proof is completed.
\end{proof}

\textbf{Theorem \ref{Theorem5}} highlights that will ultimately converge to $\mathcal{L}^*$, exhibiting a convergence radius of $\frac{\eta \epsilon^2}{2}$ within $T$ epochs.
Generalization support ensures that the OpenViewer's training loss continually approaches its optimum with increased iterations, which is crucial for handling multi-view data in open environments:
1) It demonstrates the OpenViewer's long-term generalization and stability, ensuring that even with complex data and open environments, OpenViewer can theoretically approach the optimal solution after sufficient iterations;
2) It ensures that the integration of information from different views converges to a consistent solution, reducing noise and improving the accuracy of overall decisions or predictions;
3) In open environments, when encountering unknown classes or data, OpenViewer can maintain stable generalization through appropriate adjustments;
4) It reinforces the theoretical foundation of other OpenViewer-like algorithm design, ensuring its reliability and effectiveness in practical open-setting applications.

\section{Theoretical Analysis: Complexity}\label{SupplementaryofTheoreticalTrustworthyAnalysis}

%\subsection{Complexity}\label{SupplementaryofTheoreticalTrustworthyAnalysis}

Then, we analyze the complexity of OpenViewer in each batch, including network modules and losses of the time and space complexity.

The time complexity of the proposed network modules is summarized as follows:
\textbf{RF-Module}: The primary operations involve matrix multiplications with complexities $\mathcal{O}(NC^2)$ and $\mathcal{O}(NCD_{v})$, resulting in an overall complexity of $\mathcal{O}(NC^2 + NCD_{v})$.
\textbf{CD-Module}: It includes operations such as $\mathbf{Z}_v^\top(\mathbf{X}_v - \mathbf{E}_v^{(l)})$. Thus, the total complexity is $\mathcal{O}(NCD_{v})$.
\textbf{DN-Module}: The key computations are $\mathbf{Z}_v^{(l+1)}\mathbf{D}_v^{(l+1)}$ and subtraction $(\mathbf{X}_v - \mathbf{Z}_v^{(l+1)}\mathbf{D}_v^{(l+1)})$, with a combined complexity of $\mathcal{O}(NCD_{v})$.
\textbf{CW-Fusion}: The fusion process is dominated by the centroid computation $\mathcal{O}(NCD_{v})$, inter-class distance calculation $\mathcal{O}(C^{2}D_{v})$, and the weighted fusion operation $\mathcal{O}(NCV)$, resulting in a complexity of $\mathcal{O}(NCD_{v}+C^{2}D_{v}+NCV)$.
Combining these, the total time complexity for the deep networks is $\mathcal{O}(NC(C+D_{v}+V)+C^{2}D_{v})$. 
Moreover, the space complexity of the deep unfolding networks is $\mathcal{O}(N(D_{v} + C)V)$ with the storage requirements for feature matrices, dictionaries, and fusion weights. 

The time complexity of training loss components is shown as follows:
1) $\mathcal{L}_{known}$ has a time complexity of $\mathcal{O}(N^{o}C)$, due to the cross-entropy calculation and margin distance term; 
2) The time complexity of $\mathcal{L}_{unknown}$ is $\mathcal{O}(N^{e}C + N^{e}D_{v})$, due to both the cross-entropy loss and regularization term; 
3) $\mathcal{L}_{center}$ has a time complexity of $\mathcal{O}(N^{o}C+N^{o})$ for computing the distance between samples and their class centers and updating the class centers.
In summary, the overall time complexity for the loss functions is $\mathcal{O}(N^{o}C+N^{e}D_{v})$.
Furthermore, the space complexity of the training losses is $\mathcal{O}(NC + ND_{v})$, encompassing the storage for predictions, centers, and intermediate calculations.

Overall, the time complexity of OpenViewer with $L$ layers costs $\mathcal{O}\left(\left(NC(C+D_{v}+V)+C^{2}D_{v}\right)L\right)$, and the space complexity of OpenViewer denotes $\mathcal{O}\left(N(D_{v}+C)V\right)$.

\section{Supplementary of Experiments}\label{SupplementaryofExperiments}

In this section, we present supplementary of experiments, including datasets, compared methods, and other results.

\subsection{Details of Datasets and Compared Methods}\label{DetailsofExperimentalSetup}

\subsubsection{Datasets.}\label{SupDatasetsandTestSettings}

Details for all test multi-view datasets are presented below.
Multi-feature datasets: 1) \textbf{Animals}\footnote{http://attributes.kyb.tuebingen.mpg.de/} is a deep multi-feature dataset that consists of 10,158 images from 50 animal classes with DECAF and VGG-19 features;
2) \textbf{AWA}\footnote{https://cvml.ista.ac.at/AwA/} is a large-scale dataset belonging to 50 animal categories, we extract of 6 features for all the images, that is, color histogram, local self-similarity, pyramidHOG, SIFT, colorSIFT, and SURF features;
3) \textbf{NUSWIDEOBJ}\footnote{http://lms.comp.nus.edu.sg/research/NUS-WIDE.htm} consists of 30,000 images distributed over 31 classes. 
We use five features provided by NUS, \emph{i.e.,} color histogram, color moments, color correlation, edge distribution and wavelet texture features;
4) \textbf{VGGFace2-50}\footnote{https://www.robots.ox.ac.uk/$\sim$vgg/data/vgg\_face2/} is a large-scale face recognition dataset. 
We extract a sub-dataset with 50 face categories, which consists of four features as LBP, HOG, GIST, Gabor features.
Multi-modal datasets: 1) \textbf{ESP-Game}\footnote{https://www.kaggle.com/datasets/parhamsalar/espgame} originates from an image annotation game played on a website, which contains 20,770 images, and each image is annotated by players with several descriptions. 
Among them, image features are extracted by VGG-19 networks, and text description features are extracted by BERT networks.
Here, we choose 11,032 images that are described with approximately five tags per image, and these images have a total of 7 classes;
2) \textbf{NUSWIDE20k}\footnote{https://lms.comp.nus.edu.sg/wp-content/uploads/2019/research/nuswide/NUS-WIDE.html} is real-world web image database sourced from Flickr, which consists of 269,648 social images with tags.
Meanwhile, we select a sub-dataset with 20,000 images, which contains 8 classes and each image is along with seven tags on average.

\subsubsection{Compared Methods.}\label{SupComparedMethods}

These comparison methods are described as follows.
1) \textbf{MvNNcor} \cite{Xu2020Deepembedded} modeled view-specific information and cross-correlations information through an interactive network to jointly make decisions and infer categories;
2) \textbf{TMC} \cite{Han2021Trusted} fused the uncertainty of multiple views at an evidence level with the Dempster-Shafer theory; 
3) \textbf{MMDynamics} \cite{Han2022Multimodal} introduced a trustworthy approach to multimodal classification by dynamically evaluating feature-level and modality-level informativeness;
4) \textbf{IMvGCN} \cite{wu2023Interpretable} enhanced the interpretability of neural networks by constructing a flexible graph filter and introducing orthogonal normalization; 
5) \textbf{LGCN-FF} \cite{Chen2023Learnable} integrated sparse autoencoders with a learnable GCN, enabling the simultaneous extraction of feature representations and node relationships within graphs;
6) \textbf{ORLNet} \cite{fang2024representation} derived an interpretable solution for explicit optimization representation learning objectives; 
7) \textbf{RCML} \cite{Xu24Reliable} provided decision results and attached reliabilities for conflictive multi-view data.
These multi-view models are trained and tested according to their own structure and utilize the same dataset settings as OpenViewer to test their performance in an open environment.

%\subsubsection{Implementation Details}\label{SupImplementationDetails}

\subsection{Supplementary of Experimental Results}\label{ExperimentalSupplementaryofExperimentalResults}

\begin{figure*}[t]
  \centering
  \includegraphics[width=0.9\linewidth]{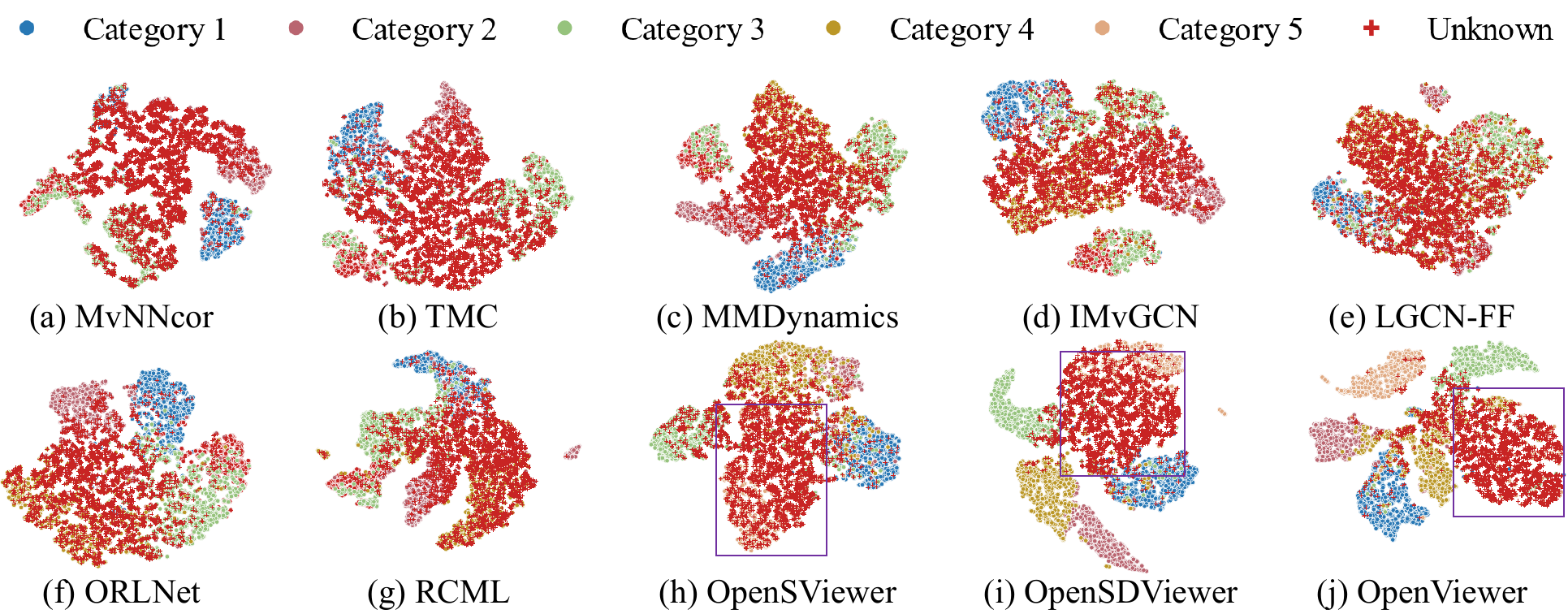}\\
  \caption{All compared methods’ t-SNE visualizations based on the representations $\mathbf{Z}$ of ESP-Game dataset.}
  \label{tnseESPGame}
\end{figure*}

\begin{figure*}[t]
  \centering
  \includegraphics[width=0.9\linewidth]{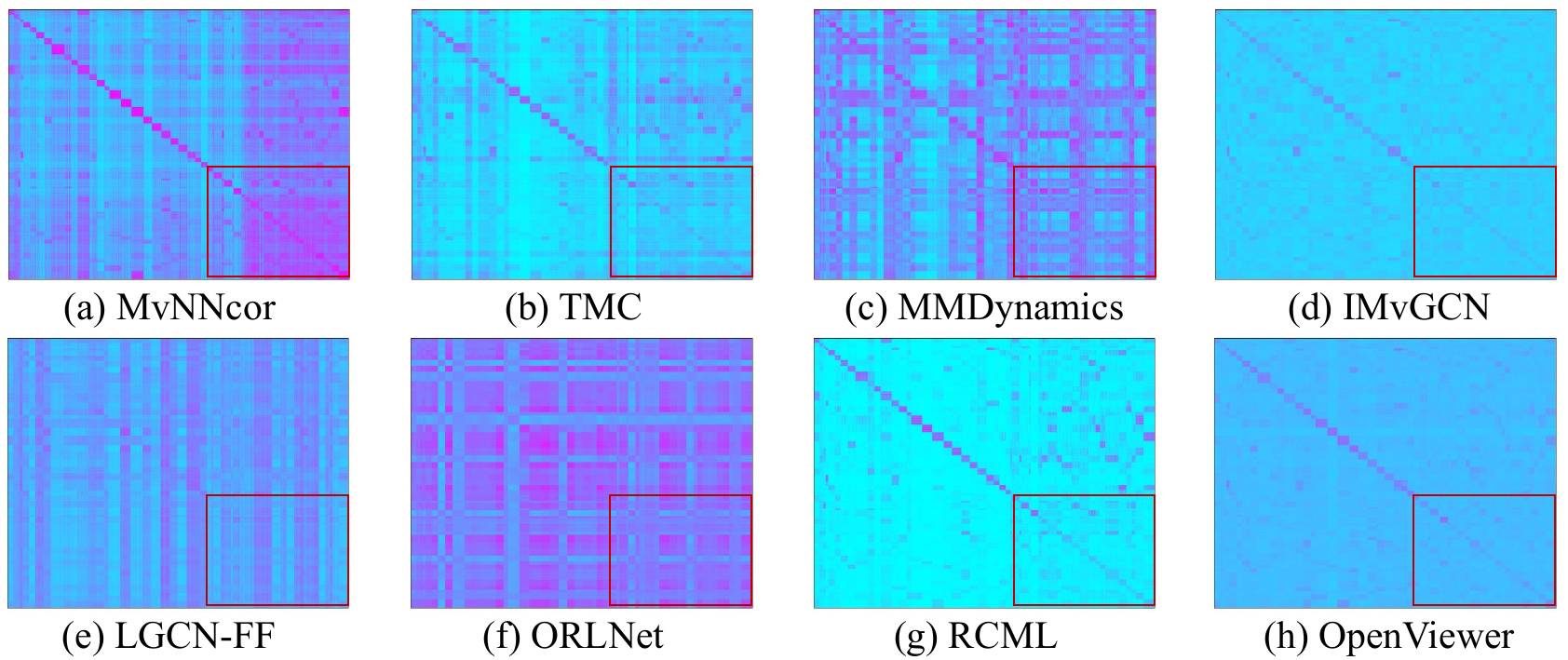}\\
  \caption{All compared methods’ heatmap visualizations based on $\mathbf{Z}\mathbf{Z}^T$ of Animals dataset.}
  \label{heatmapsAnimalssup}
\end{figure*}

Fig. \ref{tnseESPGame} supplements the t-SNE visualization results of all compared and ablation methods's representations $\mathbf{Z}$ on ESP-Game dataset.
Fig. \ref{heatmapsAnimalssup} presents the heatmap visualization results of the correlation matrices $\mathbf{Z}\mathbf{Z}^T$ for all compared methods on the Animals dataset. 
In OpenViewer's heatmap, the known parts exhibit clearly enhanced expression, while the unknown parts are effectively suppressed, preventing excessive responses to the known parts. 
In contrast, other methods like MvNNCor, MMDynamics, LGCN-FF, and ORLNet reveal noise and overconfidence between the known and unknown categories. 
The enhanced expression effect of IMvGCN is less significant across all parts. 
Among the comparison methods, RCML performs the best, but there are still some outliers that affect recognition.

\subsection{Supplementary of Parameter Sensitivity Analysis}\label{ExperimentalSupplementaryofParameterSensitivity}

\begin{figure}[!htbp]
  \centering
  \includegraphics[width=0.95\linewidth]{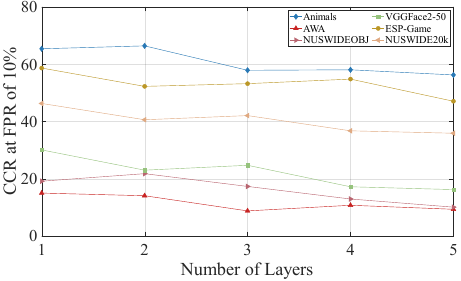}\\
  \caption{Parameter sensitivity of unfolding layers.}
  \label{Parametersensitivitylayers}
\end{figure}

\begin{figure*}[t]
  \centering
  \includegraphics[width=0.9\linewidth]{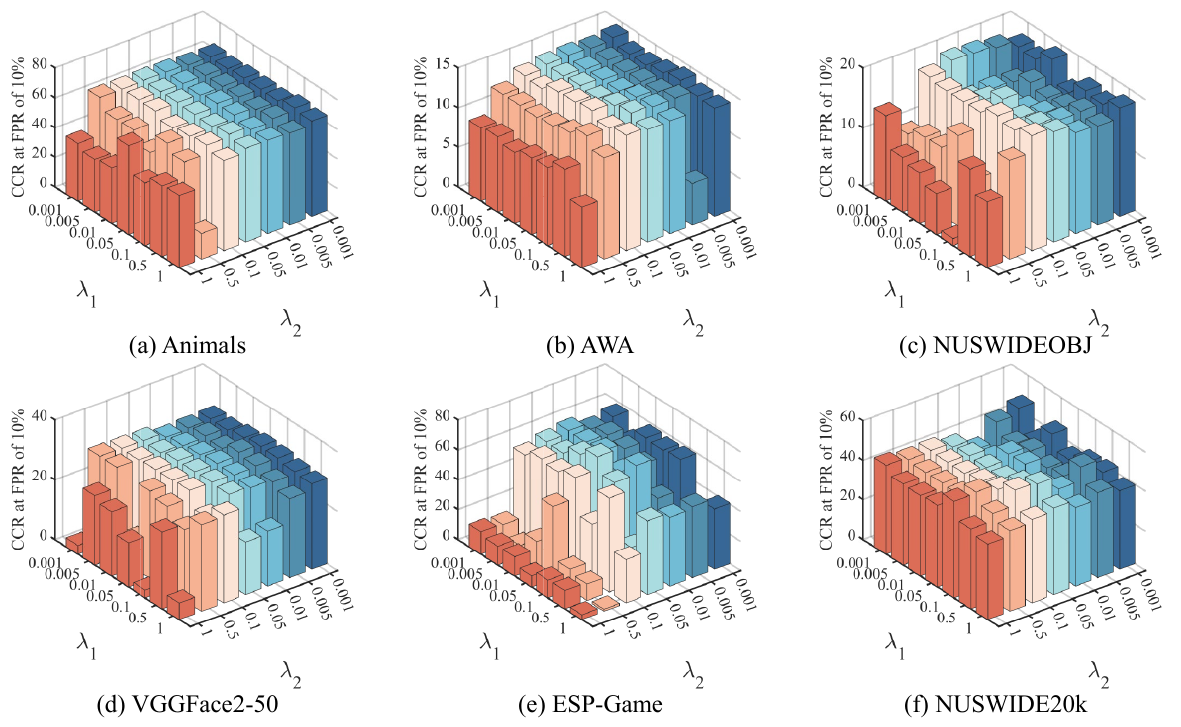}\\
  \caption{Parameter sensitivity of $\lambda_{1}$ and $\lambda_{2}$ of OpenViewer.}
  \label{Parametersensitivitylambda1lambda2}
\end{figure*}

\begin{figure*}[t]
  \centering
  \includegraphics[width=0.9\linewidth]{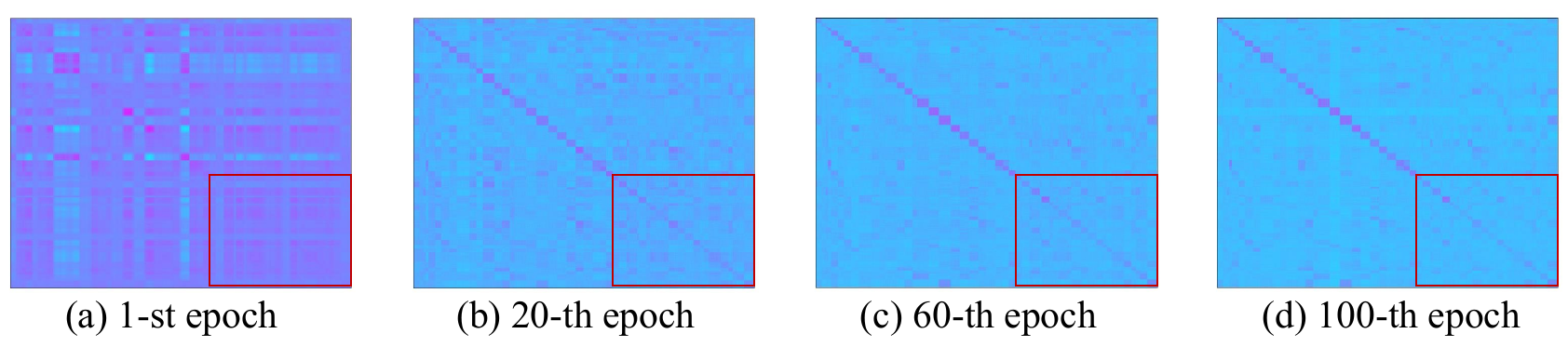}\\
  \caption{The heatmaps $\mathbf{Z}\mathbf{Z}^T$ as the changes over epochs of Animals dataset.}
  \label{heatmapAnimals2}
\end{figure*}

Fig. \ref{Parametersensitivitylayers} indicates the optimal number of unfolding layers for OpenViewer. 
The figure displays that with a single unfolding layer, the model effectively balances complexity and efficiency while maintaining interpretable expression-enhancement capabilities. 
Although performance may improve slightly on certain datasets with additional layers, the complexity also increases correspondingly.
For this reason, we set the number of unfolding layers to $L = 1$ in our experiment.
Fig. \ref{Parametersensitivitylambda1lambda2} provides all sensitivity curves for two trade-off parameters of training losses of OpenViewer, $\lambda_{1}$ and $\lambda_{2}$.
Fig. \ref{heatmapAnimals2} showcases the representations $\mathbf{Z}$ through the heatmaps of $\mathbf{Z}\mathbf{Z}^T$ across progressive training epochs.
The figure clearly displays that, initially, $\mathbf{Z}$ is saturated with noise and exhibits weak expression. 
As training epochs progress, noise is gradually eliminated, leading to an enhanced representation of known samples. 
Subsequently, the confidence of unknown samples on known classes (highlighted within the red box) are progressively suppressed. 
This progression substantiates OpenViewer's contributions to enhanced feature expression and augmented sample perception.

\end{document}